\documentclass[sn-mathphys-num]{sn-jnl}% Math and Physical Sciences Numbered Reference Style 
\usepackage{graphicx}%
\usepackage{multirow}%
\usepackage{amsmath,amssymb,amsfonts}%
\usepackage{amsthm}%
\usepackage{mathrsfs}%
\usepackage[title]{appendix}%
\usepackage{xcolor}%
\usepackage{textcomp}%
\usepackage{manyfoot}%
\usepackage{booktabs}%
\usepackage{algorithm}%
\usepackage{algorithmicx}%
\usepackage{algpseudocode}%
\usepackage{listings}%
\usepackage[normalem]{ulem}
\usepackage{wrapfig}
\usepackage{cleveref}
\usepackage{subcaption}
\usepackage{placeins}
\usepackage{thm-restate}

\def\R{\mathbb{R}}

\def\Qcal{\mathcal{Q}}

\def\Ocal{\mathcal{O}}

\def\Pcal{\mathcal{P}}

\newcommand{\Lcal}{\mathcal{L}}

\def\SSBM{\operatorname{SSBM}}
\def\diag{\operatorname{diag}}
\def\tr{\operatorname{tr}}
\def\reff{\operatorname{ref}}

\def\iter{\operatorname{iter}}
\def\fro{F}
\newcommand{\ot}{\operatorname{ot}}
\newcommand{\OT}{\operatorname{W}}

\def\bigG{\mathbb{G}}
\def\bigA{\mathbb{A}}
\def\bigV{\mathbb{V}}
\def\bigE{\mathbb{E}}

\def\eg{\emph{e.g.}}
\def\ie{\textit{i.e.}}

\def\clustermatrix{P}
\def\clustmatbar{\overline{\clustermatrix}}
\def\kbar{\overline{k}}
\def\one{{\mathbf{1}}}
\newcommand{\integ}[1]{{[\![#1]\!]}}
\newcommand{\dist}{\mathcal{D}}

\newcommand{\abs}[1]{\left\vert #1 \right\vert}

\def\-{\raisebox{.75pt}{-}} % for small minus signs

\newcommand{\add}[1]{\textcolor{black}{#1}}
\newcommand{\addblock}[1]{\color{black} #1 \color{black}}
\newcommand{\addblocknew}[1]{\color{black} #1 \color{black}}
\newcommand{\addnew}[1]{\textcolor{black}{#1}}
\theoremstyle{thmstyleone}%
\newtheorem{theorem}{Theorem}%  meant for continuous numbers
\newtheorem{lemma}{Lemma}
\newtheorem{proposition}[theorem]{Proposition}% 

\theoremstyle{thmstyletwo}%
\newtheorem{remark}{Remark}%

\theoremstyle{thmstylethree}%
\newtheorem{definition}{Definition}%

\begin{document}

\title[PASCO (PArallel Structured COarsening)]{PASCO (PArallel Structured COarsening): \\ an overlay to speed up graph clustering algorithms}


\author*[1]{\fnm{Lasalle} \sur{Etienne}}\email{etienne.lasalle@ens-lyon.fr}
%\equalcont{To use for explaining content}

\author[1]{\fnm{Vaudaine} \sur{Rémi}}%\email{remi.vaudaine@ens-lyon.fr}

\author[1]{\fnm{Vayer} \sur{Titouan}}%\email{titouan.vayer@inria.fr}

\author[2]{\fnm{Borgnat} \sur{Pierre}}%\email{pierre.borgnat@ens-lyon.fr}

\author[1]{\fnm{Gonçalves} \sur{Paulo}}%\email{paulo.goncalves@inria.fr}

\author[1]{\fnm{Gribonval} \sur{Rémi}}%\email{remi.gribonval@inria.fr}

\author[3,4]{\fnm{Karsai} \sur{M\'arton}}%\email{karsaim@ceu.edu}

% \affil[1]{ENS de Lyon, CNRS, Université Claude Bernard Lyon 1, Inria, LIP, UMR 5668, 69342, Lyon cedex 07, France}

\affil[1]{Inria, ENS de Lyon, CNRS, Université Claude Bernard Lyon 1, LIP, UMR 5668, 69342, Lyon cedex 07, France}

\affil[2]{CNRS, ENS de Lyon, LPENSL, UMR5672, F-69342, Lyon cedex 07, France}

\affil[3]{Department of Network and Data Science, Central European University, 1100 Vienna, Austria}

\affil[4]{National Laboratory for Health Security, HUN-REN Alfréd Rényi Institute of Mathematics, 1053 Budapest, Hungary}


\abstract{Clustering the nodes of a graph is a cornerstone of graph analysis and has been extensively studied. However, some popular methods are not suitable for very large graphs: \eg, spectral clustering requires the computation of the spectral decomposition of the Laplacian matrix, which is not applicable for large graphs with a large number of communities. This work introduces PASCO, an overlay that accelerates clustering algorithms. Our method consists of three steps: 
1- We compute several independent small graphs representing the input graph by applying an efficient and structure-preserving coarsening algorithm. 
2- A clustering algorithm is run in parallel onto each small graph and provides several partitions of the initial graph. 
3- These partitions are aligned and combined with an optimal transport method to output the final partition. 
The PASCO framework is based on two key contributions: a novel global algorithm structure designed to enable parallelization and a fast, empirically validated graph coarsening algorithm that preserves structural properties.
We demonstrate the strong performance of PASCO in terms of computational efficiency, structural preservation, and output partition quality, evaluated on both synthetic and real-world graph datasets.}




\maketitle

% !TEX root = ../main.tex

\section{Introduction \label{sec:introduction}}

Graphs are a fundamental tool to model modern data sets as they become increasingly complex. \add{Graphs represent} complex systems of interacting entities, and applications are found in almost all domains of science.
A pillar of graph analysis is the problem of \emph{community detection} \add{(also called \emph{clustering})} in which one wants to partition the nodes of a graph so that nodes with similar connectivity patterns are clustered~\cite{von2007tutorial,fortunato2010community}. This problem arises in various domains, such as social sciences and genomics \cite{fortunato2016community,karatacs2018application}.
This task has already been extensively studied both theoretically and practically. %The proposed tools range from spectral methods to model-based approaches. 
However, these algorithms are often unsuited for large-scale community detection problems where the number of nodes $N$ and communities $k$ can become prohibitive. 
%Spectral methods scale poorly with these quantities (\eg, spectral decomposition`s complexity is $\Ocal(kN^2)$). 
%\el{Model-based methods also suffer from  scaling issues, as they need to optimize a function (\eg, the likelihood) which computation often requires at least $\Ocal(N^2)$ operations, over the set at least as big as the set of possible partitions (growing in $\Ocal(k^N)$).}  
%Modularity maximization methods, such as the Louvain algorithm, have better scaling properties but have issues concerning the output partition, \eg, resolution limit \cite{fortunato2007resolution, fortunato2016community}, or finding spurious communities \cite{peixoto2023} 
%\tv{ref ici pour le prouver ? ou alors à enlever? on parle de scaling issue pas de perf issue ?} \el{\cite{fortunato2007resolution} Pierre tu penses à d'autres réf?}. \PB{On peut ajouter \cite{peixoto2023} et \cite{fortunato2016community} qui donnent une vision plus récente de 2 des limites d'intreprétation de Louvain: question de résolution dans Fortunato ; question d'inférence avec modularité pour Peixoto.}

Several avenues have been explored to solve these scaling issues. Most follow this general scheme: first, reduce the size of the input graph, then cluster the reduced graph, and finally export the partition of the reduced data to the original data. 
There are two dominant ways to reduce the input data size: sampling or coarsening. \add{ While sampling techniques select a subset of nodes and discard the others, coarsening methods aggregate groups of nodes to obtain smaller graphs with similar overall structures. }
% For instance, sampling the nodes for spectral clustering has been reviewed in \cite{tremblay2020approximating}, followed by accelerated computation of the graph spectrum. 
%More generally, when the graph is not the data set itself but some underlying structure (\eg, as from Gaussian graphical models), sub-sampling can be directly used to speed up computations before even constructing the graph~\cite{fowlkes2004spectral,shinnou2008spectral,liu2013large,chitta2012efficient}.
%However, in this article, we are interested in the case where the input data is {\em the graph itself}, hence we exclude the latter framework.
% However, most approaches are coarsening-based. They reduce the graph into a smaller graph with similar characteristics and structure. Additionally, a particular attention has been devoted to accelerate spectral clustering, especially the step of spectral decomposition, see \cite{tremblay2020approximating}. 
% Another way is to use so-called {\em clustering ensemble} techniques\cite{vega2011survey,ghosh2018cluster}, that compute several partitions of potentially limited quality, using off-the-shelf methods, and combine them into a final clustering of better quality. However, the speedup depends on the existence of pre-existing fast clustering methods. In this sense, clustering ensemble techniques allow to enhance partitions quality from fast clustering algorithms rather than accelerate clustering algorithms.

The present article proposes a new coarsening-based algorithmic overlay to reduce the overall computation time of clustering procedures. 
We focus on undirected networks and develop a versatile framework \emph{that can be used with any chosen clustering method}.
The method consists of three main parts and two novel contributions are proposed. First, the \emph{coarsening} phase computes several simpler and smaller representations of the input graph. We derive a new fast and empirically structure-preserving algorithm based on random edge contractions. 
The algorithm is executed multiple times in parallel to generate several simplified representations of the input graph. Then, in the \emph{clustering} part, any user-specified clustering algorithm adapted to weighted undirected graphs can be run in parallel on these simple graphs. Finally, after lifting the partitions of the coarsened graphs to obtain partitions of the input graph, we \add{move} to the \emph{fusion} part. Using an optimal-transport-based method, we combine these partitions to produce a better and final partition of the input graph.

\subsection{Contributions}
\begin{itemize}
    \item We propose PASCO, a new three-step coarsening-based framework to speed up graph clustering algorithms. Innovation comes from the structure of the \addnew{pipeline} that computes many differently coarsened graphs before clustering them independently (see \Cref{fig:pasco_pipeline}).
    %, lifts the resulting partitions to partitions of the input graph and then combines the results to output a single partition by leveraging an alignment procedure of the partitions (see \Cref{fig:pasco_pipeline}). 
    It is a flexible design and serves as a computational overlay that can be applied to any clustering algorithm.
    \item We design a fast and efficient random coarsening algorithm \addnew{as a key component of the PASCO clustering pipeline}. Our approach is opposed to classical coarsening-based clustering approaches that rely on convoluted, and often costly, coarsening mechanisms. 
    \item We extensively evaluate PASCO and its components. The coarsening step and the fusion step are analyzed to confirm the preservation of the structure and the increase in partition quality. Then, the entire PASCO pipeline is tested on synthetic and real graph data. The results show speedups for computationally heavy clustering methods, while maintaining or even improving quality on complex real-world networks.
\end{itemize}

\subsection{Related Works}

Clustering the nodes of a graph has attracted a lot of attention: spectral methods~\cite{von2007tutorial}, information-theoretic approaches \cite{infomap:2011}, model-based approaches \cite{peixoto2014efficient}, and the popular maximization of modularity \cite{blondel2008fast}\add{, which is a measure of the quality of a partition of the nodes of a graph into communities.}
% Spectral methods such as Spectral Clustering (SC) \cite{fiedler1973algebraic,donath1973lower,ng2001spectral,von2007tutorial} exploit the link between the community structure and the leading eigenvectors of various matrix representations. %The consistency of this approach under the Stochastic Block Model (SBM) has been investigated in \cite{von2008consistency}. 
% Model-based approaches have also been considered, %. Assuming an underlying Stochastic Block Model (SBM), the model parameters can be inferred through maximum likelihood estimation, but 
% with maximum \emph{a posteriori} based estimators \cite{snijders1997estimation,nowicki2001estimation}, as well as MCMC methods \cite{peixoto2014efficient}. % A method based on an optimized MCMC approach with a selection criterion for the number of communities guided by the Minimum Description Length (MDL) has been proposed in . 
% Other approaches try to optimize some quality criterion. An example of this is the maximization of modularity, which yields the famous Louvain algorithm \cite{blondel2008fast} or its more recent variant, the Leiden algorithm \cite{traag2019louvain}.  
We refer the reader to \cite{fortunato2016community} for reviews on community detection methods.
However, these methods do not always scale well. Hence, various works have been proposed to speed up clustering computations. Some of these approaches are detailed now.

\textbf{General fast approaches to clustering:}
Substantial work has been devoted to accelerating spectral clustering, where the efforts essentially focus on faster solving of the spectral decomposition, \eg, using the Nyström method \cite{pourkamali2020scalable} or the power method \cite{boutsidis2015spectral}. 
In \cite{tremblay2016compressive}, the authors tackle the high computational cost of spectral clustering by approximating the spectral embedding using an efficient graph filtering of random signals and accelerating the $k$-means part using a sub-sampling strategy. Another way to accelerate clustering is to reduce the number of edges in the graph before computing the clustering. To do so, several sparsification techniques have been proposed, either by sampling and removing random edges \cite{spielman2011spectral} or using effective resistance \cite{spielman2008graph}. The review \cite{tremblay2020approximating} provides an overview of acceleration techniques in the case of spectral clustering. Other fast approaches construct a bipartite graph between the initial set of nodes and a new and smaller set of nodes and recover the community structure of the input graph from this bipartite graph \cite{yan2009fast, li2015large}.

\textbf{Coarsening approaches:} Most coarsening approaches \cite{hendrickson1995multi,dhillon2007weighted,loukas2019graph} rely on an iterative multilevel edge-contraction-based coarsening algorithm. That is, several coarsened graphs of decreasing sizes are computed iteratively. At each coarsening level, several edges are selected and collapsed to put their end vertices into the same hypernode. Then, some clustering algorithm is run {\em on the smallest coarsened graph} before {\em lifting} the result iteratively back to the next larger set of nodes. 
%At each lifting step, the partition is refined by evaluating the gain (w.r.t. a certain cost) of moving each node to a different cluster. 
%Below, we quickly survey approaches that follow this mechanism as this will motivate one of our contributions, a fast and efficient random coarsening alternative. For more details on coarsening-related works, see \Cref{subsec:classical_coarsening}.
%We see that the tendency of existing works is to develop more elaborated, and thus more computationally heavy, coarsening steps. 
At each level, existing approaches exploit mainly one coarsening process. Therefore, at each lifting step, the partition is refined by evaluating the gain (w.r.t. a certain cost) to obtain a satisfying final partition. 
There is typically a trade-off between the extent of graph simplification used to accelerate clustering and the resources required to recover an accurate partition.
Our new coarsening algorithm is designed to prioritize efficiency in this trade-off: the quality of the partition will be ensured by its insertion into our three-step framework and, in particular, the fusion of clusters obtained from {\em multiple} coarsened graphs.
\add{Other coarsening approaches, not especially designed for accelerating clustering, also exist. Most of them focus on preserving the spectral properties of the graph as in \cite{bunimovich2019finding,jin2020graph}, that is, distorting as little as possible the eigenvalues and/or eigenvectors of matrix representations of the graphs (\eg, adjacency matrix, Laplacian matrix). For an overview of existing coarsening methods, we refer the reader to the survey \cite{chen2022graph}.}
\addnew{Multilevel clustering algorithms also contain a part of coarsening, as nodes are being grouped. Often, they perform deterministic optimized grouping according to a given criterion. In some cases, one wants to obtain balanced groups \cite{gottesburen2021deep}. Often, groups are made to optimize a clustering quality measure, \eg, the modularity \cite{blondel2008fast, rotta2011multilevel, waltman2013smart} or the description length \cite{infomap:2011}. Our PASCO clustering pipeline fits into this class of multilevel approaches. However, it differs from the mentioned works as its coarsening phase is randomized and is not guided by a specific clustering criterion, and several coarsened graphs are computed to obtain a better final partition.}

\textbf{Clustering ensemble:} Clustering ensemble combines multiple results of clustering {\em the same graph} to form a more robust consensus, improving stability and reliability by aggregating diverse partitions from different off-the-shelf algorithms or parameter settings. PASCO can be framed within the clustering ensemble framework, as we obtain several partitions of the initial graph (by random coarsening, clustering, and lifting) and combine them to output a final partition. Although both approaches involve merging multiple partitions, the philosophy is different from the usual clustering ensemble techniques: we first aim to {\em accelerate} clustering and not especially enhance the final clustering quality in terms of stability and robustness.
%We mention here some classical fusion methods from clustering ensemble.
Overall, clustering ensemble methods can be divided into two main categories \cite{ghosh2018cluster}. 
The first is based on consensus functions where the output clustering is the one optimizing a notion of agreement of the given partitions \cite{wu2014k}, while
the second constructs a co-association matrix that characterizes the similarity between the data items based on the partitions \cite{fred2001finding}.

\subsection{Outline of the paper and notations}

% We start by providing a general description of the proposed method, called PASCO, in \Cref{sec:pasco}, where we expose the general framework and how the PASCO pipeline works. In the next two sections, we dive into the details of the most innovative phases of our methods. First, \Cref{sec:coarsening_details} deals with the \emph{coarsening} part. We provide information on generic coarsening frameworks before explaining in details the specificity of our approach. In \Cref{sec:align_details}, we describe how the alignment and fusion step is performed, based on an optimal-transport-based method. Finally, in \Cref{sec:expe}, we present the results of our numerical experiments to assess the performances of PASCO. 
%\paragraph{Outline of the paper and notations} 
The general framework of PASCO is introduced in \Cref{sec:pasco}. Its key phases are then further explained. The coarsening is detailed in \Cref{sec:coarsening_details} while alignment and fusion are presented in \Cref{sec:align_details}. The experimental results are shown in \Cref{sec:expe}.

%\subsection{Notations \label{sec:notations}}

For any integer $n\geq1$, we denote by $\one_n$ the vector of $\R^n$ with all entries equal to 1. The set of integers ranging from 1 to $n$ is denoted by $\integ{n}$. 
We will use exponents $G^{(\ell)}$, $1 \leq \ell \leq c$ to denote sequences of $c$ coarsened graphs, while the index $r$ in $G_{r}, 1 \leq r \leq R$ denotes the output of $R$ independent instances of the randomized coarsening algorithm.

% !TEX root = ../main.tex

\section{The PASCO approach for clustering \label{sec:pasco}}

Our approach aims to speed up clustering computations by applying a given clustering algorithm to several reduced versions of the initial graph and then combining the results to output the final clustering. 
%This allows us to circumvent the high computational complexity of some clustering algorithms by feeding them with smaller graphs while still maintaining good clustering quality by leveraging the information of the various coarsened graphs. 
%Implementing this approach requires the coarsening phase and the fusion phase to be fast and efficient, so that the gains on the clustering phase are not counterbalanced by the extra computations. \tv{ça c'est un peu déjà dit avant: l'enlever ? ou le merger avec avant ?}

\begin{figure}
\centering
\includegraphics[width=0.8\linewidth]{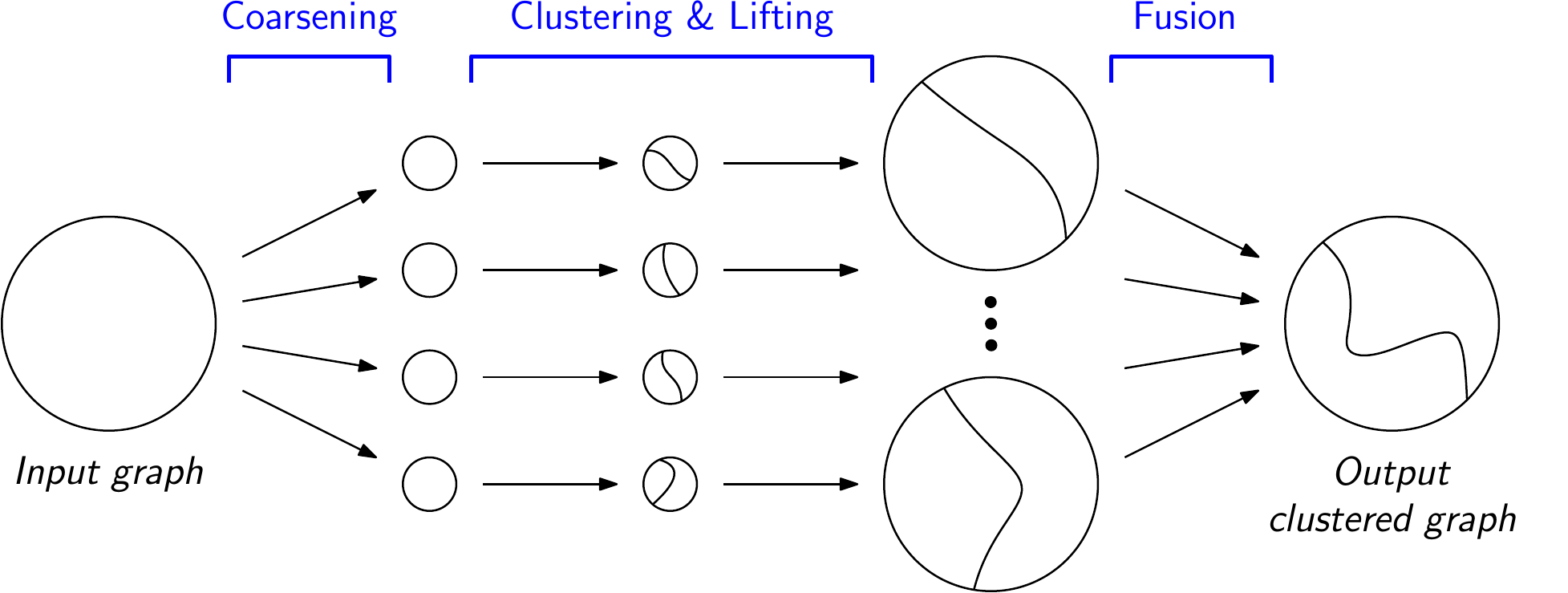}
\caption{PASCO pipeline. \add{Disks represent graphs. Curved lines in the disks figure the separations of the nodes into clusters.} {\em Coarsening}: Apply \add{our random, fast and information-preserving} coarsening algorithm to compute several small graphs, \add{each one being a small-size representation of the input graph}. 
{\em Clustering \& Lifting}: \add{Apply any off-the-shelf clustering algorithm \emph{in parallel} to each} of the coarsened graphs and lift each partition to a partition of the input graph. {\em Fusion}: Combine the partitions to output the final partition; \add{it is done by finding the partition that best agrees with all the given partitions, through solving a optimal-transport problem}.}
\label{fig:pasco_pipeline}
\end{figure}

%We developed such a coarsening algorithm that is both fast and empirically information preserving. 
Given some initial graph $\bigG = (\bigV,\bigE)$ with vertex set $\bigV$ and edge set $\bigE$, the \emph{random} coarsening algorithm is run $R$ times to obtain the coarsened graphs $G_{1}, \cdots, G_{R}$. A clustering algorithm is then applied to each of these graphs. The resulting partitions of the nodes of $G_{1}, \cdots, G_{R}$ are lifted up to partitions of the nodes of $\bigG$ and then combined to retrieve as much information as possible and output a final clustering. See Figure~\ref{fig:pasco_pipeline} for a schematic illustration of our approach. 
Below, we provide an overview of each part of the pipeline (coarsening, clustering, alignment, and fusion). The reader can refer to the next sections for more details.

\FloatBarrier
% \subsection{Coarsening \label{subsec:coarsening}}

\textbf{Coarsening:} We propose a new randomized coarsening algorithm that takes into account the structure of the initial graph. %so that the resulting coarsened graph tends to preserve most of the information about the community structure of the initial graph. 
% Here, we provide a brief overview of this algorithm.
This algorithm adopts a multilevel approach where we create the sequence of incrementally coarsened graphs
% \footnote{We use exponents $G^{(\ell)}$, $1 \leq \ell \leq c$ to denote sequences of coarsened graphs, while indices $G_{r}, 1 \leq r \leq R$ denote the output of $R$ independent instances of the randomized coarsening algorithm.} 
% PB: I have put this remark under the notations sub-section
$\bigG = G^{(0)}, \dots, G^{(l)}, \dots, G^{(c)} = G$, starting from the initial graph $\bigG$ of size $N$. Each graph $G^{(\ell+1)}$ is obtained by coarsening $G^{(\ell)}$ to reduce the number of nodes from $n^{(\ell)}$ to $n^{(\ell+1)}$, such that $n^{(\ell)} > n^{(\ell+1)}$. The number of coarsening steps $c$ is the one required to reach the small target size $n$. 
Each iterative coarsening (from $G^{(\ell)}$ to $G^{(\ell+1)}$) is based on an edge-contraction approach. 
Our strategy is to sample edges (according to a given rule) and contract them by putting the two end-vertices into the same ``hypernode'', as shown in Figure~\ref{fig:one_step_coarsening}. 
We repeat this procedure until the target size $(n)$ of the coarsened graph is reached or no edge is available (according to our sampling rule).  
Note that coarsening generates small weighted graphs with self-loops.
The challenge of this approach is to find a relevant sampling rule so that the coarsening algorithm is both fast and as information-preserving as possible. In \Cref{sec:coarsening_details}, we provide all the details for this coarsening step, including details on the sampling rule, its positioning with respect to the state of the art, and the properties of the coarsening.

\begin{figure}
\centering
\includegraphics[height=4.5cm]{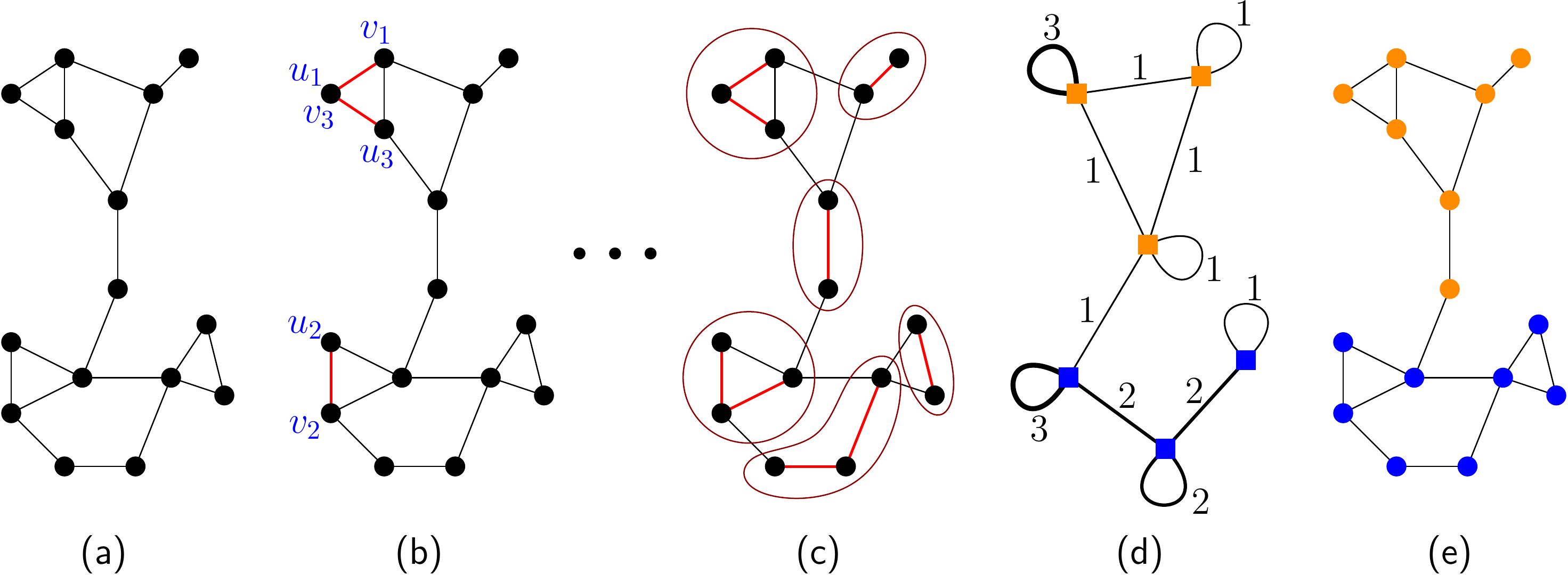}
\caption{\label{fig:one_step_coarsening} Illustration of one coarsening iteration (\Cref{alg:coarsening_one}). (a) The original graph; (b) The first three pairs of sampled nodes (in blue) and their corresponding edges (in red): first, $u_i$ is sampled uniformly at random among the unvisited nodes, then $v_i$ is sampled among the neighbors of $u_i$ (no restriction on $v_i$). (c) The sampled edges at the end of the sampling phase; each set of vertices connected by sampled edges is circled and yields a hypernode. \add{These circles are a visual representation of the coarsening table, see \Cref{def:coarsening_table}.} (d) The coarsened graph. Squares represent hypernodes and edge weights are given and represented by edge thickness (note the presence of self-loops). Node color represents a possible partition. \add{(e) Result of lifting the partition from the coarsened graph to the original graph.}}
\end{figure}

% \subsection{Clustering \label{subsec:clustering}}

\textbf{Clustering:} The clustering phase consists in finding a partition of the hypernodes for each of the $R$ coarsened graphs $G_{1}, \cdots, G_{R}$ (see \Cref{fig:pasco_pipeline}). Interestingly, we can operate independently on each graph and compute in parallel these partitions to accelerate computation.
%This can be done as the clustering is performed on the coarsened graphs that are independent of each other. 
Our pipeline is designed so that any clustering algorithm can be used as long as it handles undirected weighted graphs. 
%In particular, we do not consider clustering algorithms that only deal with unweighted graphs.
However, we only focus on algorithms that generate non-overlapping partitions. 
We point out that our fusion part \add{can handle} partitions with different numbers of clusters. Therefore, PASCO can be used with clustering algorithms that automatically choose the number of communities. 
%We distinguish between algorithms that require a user-specified number of communities and those that automatically choose the output number of communities. In the first case, all output partitions will have the same number of clusters, while in the other case, various numbers of clusters may arise, and the fusion part (see the next subsection) will cope with this disparity.
%\paragraph*{Lifting} 
At this stage, we obtain a partition of the hypernodes for each coarsened graph $G_{r}$. These partitions are then \emph{lifted} to partitions of the nodes of the initial graph: each initial node inherits the class of the hypernode to which it belongs \add{(see \Cref{fig:one_step_coarsening}e)}. 
%See \Cref{subsec:lifting_part} \tv{ref break} for the mathematical formulation of lifting partitions.

% \subsection{Alignment and Fusion \label{subsec:label&fusion}}
\textbf{Alignment and Fusion:}
The final step of PASCO is to combine the various partitions lifted to the original graph into a single output partition. This is inspired by methods of ensemble clustering. %\cite{vega2011survey, alqurashi2019clustering}. 
To do so, we propose to achieve consensus among multiple partitions by leveraging optimal transport (OT)  \cite{peyre2019computational}. We first briefly define the partition matrices that are used to encode partitions.

\begin{definition}[Partition matrix]
\label{def:partition_matrix}
A matrix $\clustermatrix \in \{0,1\}^{N \times k}$ is a \emph{partition matrix} if it is column-stochastic $\clustermatrix \one_k = \one_{N}$. The fact that $P_{ij}=1$ indicates that node $i$ is attributed to cluster $j$.
We denote by $\mathcal{P}_{N, k}$ the set of partition matrices.
\end{definition}

Let ${\clustermatrix}_1, \cdots, {\clustermatrix}_R$ be the partition matrices representing the $R$ partitions of the initial graph such that ${\clustermatrix}_r \in \mathcal{P}_{N,k_r}$ with $k_r$ being the number of clusters in the $r$-th partition. 
There are several challenges that need to be tackled in order to obtain a consensus partition from $(\clustermatrix_r)_{r \in \integ{R}}$. First, these partitions may not have the same number of clusters as some clustering algorithms infer the number of clusters. Second, the partitions may not be consistent with each other, and even if they are, it is necessary to identify the unknown correspondences between their clusters\footnote{For example, permuting the columns yields different representations of the same partition.}.
To overcome these challenges, a core idea from the literature is to find a reference partition $\clustmatbar \in \Pcal_{N ,\kbar}$ which is the ‘‘closest'' to all the partitions $\clustermatrix_1, \cdots, \clustermatrix_R$.  
Optimal Transport (OT) provides tools to align probability distributions according to a ‘‘least-effort'' principle. It can be used to measure a notion of similarity between partitions.
%, as similarly done in \cite{li2019optimal}. 
Given a fixed prescribed number of clusters $\kbar$, we solve the OT barycenter problem
\begin{equation}
\label{eq:ot_barycenter}
\underset{\clustmatbar \in \mathcal{P}_{N, \kbar} }{\min}  \frac{1}{R} \sum\limits_{i=r}^R \OT^2_2(\mu_{\clustermatrix_r}, \mu_{\clustmatbar})\,,
\end{equation}
where $\mu_{\clustmatbar}, \mu_{\clustermatrix_r}$ are discrete probability distributions associated \add{with} $\clustmatbar, \clustermatrix_{r}$, $\OT^2_2$ is the squared Wasserstein distance and $\mathcal{P}_{N, \kbar}$ is the set of partition matrices (\Cref{def:partition_matrix}). 
%See \Cref{sec:align_details} for more details. 
%
In practice, to solve this barycenter problem, the algorithm starts from an initial reference $\clustmatbar$, and then alternates between realigning the partitions to this reference (\emph{alignment step}), and updating this reference (\emph{fusion step}) until convergence.
We provide all the details about this alignment and fusion step in \Cref{sec:align_details}.

\section{Coarsening in PASCO: contributions \label{sec:coarsening_details}}

This section details the implementation of a coarsening-based clustering method, reviews classical coarsening approaches, highlights PASCO’s design for enhanced speed, and conjectures a phase transition in stochastic block model parameters when PASCO yields good performance.

% In this section, we describe how to implement a coarsening-based method and present various classical coarsening approaches for clustering. 
% Then, we present the specificity of PASCO and key design choices made to enhance its speed. 
% %Finally, we identify scenarios where PASCO is expected to perform well. 
% Finally, we conjecture a phase transition on the stochastic block model parameters indicating when PASCO would yield good performance.

\subsection{General principles of coarsening methods \label{subsec:classical_coarsening}}

Let us present the general principles that are shared by classical coarsening methods.
%First, let us give some notations and terminology. 
Coarsening is encoded through \emph{coarsening tables}, which are arrays indicating to which hypernode each node is associated, as formalized in this definition.
\begin{definition}[Coarsening Table]
For a graph $\bigG = (\bigV,\bigE)$ with vertex set $\bigV = \{u_1, \dots, u_N\}$, coarsened into a graph $G = (V,E)$ with $V = \{ u'_1, \dots , u'_n \}$, the coarsening table is the vector $h \in \integ{n}^N$, such that node $u_i \in V$ is associated \add{with} hypernode $u'_{h_i} \in V$. We can also encode this table $h$ into a {\em coarsening matrix} $H \in \{0,1\}^{N \times n}$, where $H_{i,j}=1$ if and only if node $u_i$ is associated \add{with} hypernode $u'_j$. 
\label{def:coarsening_table}
\end{definition}

\begin{wrapfigure}{R}{0.5\linewidth}
	\vspace{-3em}
	\begin{minipage}{\linewidth}
	\begin{algorithm}[H]
		\addblock{
		\caption{\label{alg:coarsening_global} Global coarsening}
	    \begin{algorithmic}[1]
	    \State Input: Adjacency matrix $\bigA$ of graph $\bigG$ of size $N$. 
	    \item Compression factor $\rho$. 
	    \State Target graph size $n = \lfloor N / \rho \rfloor$.
	    \State $A^{(0)} \leftarrow \bigA$
	    \State $h \leftarrow (1, \dots, N)$
	    \State $\ell \leftarrow 0$
	    \While{$A^{(\ell)}$ has more than $n$ nodes}
	     \State Coarsen $A^{(\ell)}$ to obtain $h^{(\ell)}$ and $A^{(\ell+1)}$ with \Cref{alg:coarsening_one}
	     \State Update $h$ given $h^{(\ell)}$
	     \State $\ell \leftarrow \ell+1$
	    \EndWhile \\
	    \Return $A^{(\ell)}$, $h$
	    }
	    \end{algorithmic}
	\end{algorithm}
	\end{minipage}
	\vspace{-2em}
\end{wrapfigure}

We recall that coarsening is usually done by constructing a sequence of incrementally coarsened graphs $\bigG = G^{(0)} , \dots , G^{(c)} = G$ starting from a graph size $N$ down to the target size $n$. The target size is defined by $n = \lfloor N/\rho \rfloor$, where $\rho$ is called the \emph{compression factor} and is a hyper-parameter of the coarsening method. 
This main scheme is detailed in \Cref{alg:coarsening_global}, where $h$ is the coarsening table from the initial graph $\bigG$ to the current most coarsened graph $G^{(\ell)}$. 
When one coarsening step is performed (step 8), the coarsening table $h^{(\ell)}$, from $G^{(\ell)}$ to $G^{(\ell+1)}$, is obtained. Hypernodes are then relabeled so that $h^{(\ell)}$ takes consecutive integer values starting at 1. Then, the next coarsened graph $G^{(\ell+1)}$, or rather its adjacency/weight matrix  $A^{(\ell+1)}$, can be computed using the adjacency matrix $A^{(\ell)}$ of graph $G^{(\ell)}$ and the coarsening matrix $H^{(\ell)}$ encoding $h^{(\ell)}$ according to $A^{(\ell+1)} = {H^{(\ell)}}^\top A^{(\ell)} H^{(\ell)}$.
% \begin{equation}
% A^{(\ell+1)} = {H^{(\ell)}}^\top A^{(\ell)} H^{(\ell)} \,.
% \label{eq:adj_coarsening}
% \end{equation} 
Then $h$ is updated coordinate-wise using by $h_i \leftarrow h^{(\ell)}_{h_i}, \quad \forall i \in \integ{N}$.

The diversity in graph coarsening methods arises from various sampling strategies for selecting collapsing edges. In \cite{hendrickson1995multi}, edges are contracted by randomly selecting an unvisited node and an unvisited neighbor. The heavy-edge heuristic introduced in \cite{karypis1998fast} prioritizes edges with the heaviest weights, aiming to group similar nodes into the same hypernode. This approach has been extended with tailored weights to optimize specific objectives, such as Graclus \cite{dhillon2007weighted} for cut optimization and \cite{loukas2019graph} for preserving spectral properties.

After coarsening the initial graph $G^{(0)}$ into $G^{(c)}$, a clustering algorithm is run to obtain a partition $\clustermatrix^{(c)}$ of the nodes of $G^{(c)}$. This clustering information is then transferred from the coarsened graph to the initial graph using a so-called \emph{lifting} step. 
A simple way to lift a partition $\clustermatrix^{(\ell)}$ of $G^{(\ell)}$ to a partition of $G^{(\ell-1)}$ is to state that each node in $G^{(\ell-1)}$ inherits the cluster of the hypernode of $G^{(\ell)}$ to which they belong. 
Mathematically, this translates to the matrix product $\clustermatrix^{(\ell-1)} = H^{(\ell-1)} \clustermatrix^{(\ell)}$, where $H^{(\ell-1)}$ is the coarsening matrix from $G^{(\ell-1)}$ to $G^{(\ell)}$. 
However, for classical coarsening approaches, this simple lifting method does not provide good quality clustering as the coarsening loses too much information and extra refining steps are necessary.
%Therefore, they proceed by refining the lifting procedure at each step by evaluating the gain (w.r.t. a certain cost) of moving each node to a different cluster \cite{kernighan1970efficient,fiduccia1988linear}. 
%Therefore, if $\operatorname{refine}(P,G)$ denotes the partition obtain after refining the partition $P$ of graph $G$, the iterative lifting procedure of existing methods would write as
% \begin{equation}
% \clustermatrix^{(\ell-1)} = \operatorname{refine}(H^{(\ell-1)} \clustermatrix^{(\ell)},G^{(\ell-1)}) \,. 
% \label{eq:lifting_related_work}
% \end{equation}

Overall, existing approaches often focus on complex, computationally intensive coarsening steps. Additionally, by exploiting only one coarsening process, they are bound to make use of computationally costly refinement steps in the lifting procedure to recover a satisfying partition. 
In the next section, we will see how PASCO differentiates itself from these existing works by resorting to simpler (and thus faster) coarsening and lifting steps. 
Partitions of good quality will be recovered, not by complexifying the procedure, but by using several coarsening processes in parallel and combining the resulting partitions (see \Cref{sec:align_details}). 

\subsection{Coarsening in PASCO}
\label{subsec:coarsening_implem}

The coarsening approach in PASCO is similar to some existing methods in the sense that it is an iterative and multi-level edge-contraction-based
coarsening method. Starting from the initial graph $\bigG$ of size $N$, we aim to coarsen it to a smaller graph of target size $n = \lfloor N/\rho \rfloor$,  ($\rho$ being the \emph{compressive factor}) following \Cref{alg:coarsening_global}. 
The innovation for PASCO comes from the way each iterative coarsening (from $G^{(\ell)}$ to $G^{(\ell+1)}$) is performed (step 8 in \Cref{alg:coarsening_global}). More precisely, as it is an edge-contraction-based approach, we introduce a new simple but efficient edge sampling mechanism, detailed in \Cref{alg:coarsening_one}. In this approach, we propose to sample uniformly at random an unvisited vertex $u$ of $G^{(\ell)}$, and sample one of its neighbors $v$ (potentially already visited) uniformly at random (steps 5 and 6 of \Cref{alg:coarsening_one} \add{or \Cref{fig:one_step_coarsening}b}). The edge $(u,v)$ is used \add{to update the coarsening table}; that is, $u$ and $v$ are \add{assigned} to the same hypernode (step 7), and then both vertices $u$ and $v$ are set as visited (step 8). \add{Observe that groups of nodes circled in dark red in \Cref{fig:one_step_coarsening}c represent the final state of the coarsening table.}

\begin{algorithm}[H]    
	\addblock{
    \caption{\label{alg:coarsening_one} One level of coarsening}
    \begin{algorithmic}[1]
     \State Input: Current adjacency matrix $A^{(\ell)}$, target graph size $n$.
     \State $V_a \leftarrow \{1, \dots, n^{(\ell)}\}$ \Comment{Initialize the set of available nodes.}
     \State $h^{(\ell)} \leftarrow (1, \dots, n^{(\ell)})$  \Comment{Initialize the coarsening table.}
     \While{$V_a \neq \emptyset$} \Comment{while there are available nodes}
         \State Choose $u$ a node in $V_a$ uniformly at random
         \State Choose $v$ a neighbor of $u$ uniformly at random
         \State $h^{(\ell)}_u \leftarrow h^{(\ell)}_v$ \Comment{Put $u$ and $v$ into the same hypernode}
         \State $V_a = V_a \backslash \{u,v\}$ \Comment{Remove $u$ and $v$ from the available nodes}
         \State \textbf{If} target size is reached \textbf{do} \textbf{break}
     \EndWhile 
     \State Relabel $h^{(\ell)}$ and compute $A^{(\ell+1)}$ \\
     \Return $A^{(\ell+1)}$, $h^{(\ell)}$.
    \end{algorithmic}
     }
\end{algorithm} 

\textbf{Computational efficiency:} First, the algorithm aims at minimizing the number of intermediate coarsening steps $c$ by creating hypernodes that contain as many nodes as possible at each step. 
In \cite{hendrickson1995multi}, the authors proposed to sample $u$ and $v$ from the set of unvisited nodes, restricting the hypernodes to contain at most two nodes. As a consequence, the coarsening step quickly runs out of available edges to collapse and a new intermediate coarsened graph must be computed. To avoid this issue and create bigger hypernodes, we relax the restriction \add{on} unvisited nodes: we only require that $u$ is unvisited and we \emph{put no restriction on $v$}. 
%For that, the first elementary approach is to use coarsening tables inside Algorithm~\ref{alg:coarsening_one} instead of updating the graph at each edge collapse (remark that this is not specific to PASCO). 
%Additionally, in Algorithm~\ref{alg:coarsening_one}, by selecting $u$ as an unvisited node and \emph{putting no restriction on $v$}, we allow the while-loop to terminate later and thus reducing the number of updates of the adjacency matrix \eqref{eq:adj_coarsening}. This is to put in perspective with \cite{karypis1998fast}, which is very similar to PASCO but imposes $v$ to also be an unvisited node, preventing collapsing edges from sharing a common node. But as a consequence, t
Moreover, the sampling of collapsing edges by first taking a node uniformly at random is very efficient, as it can be done in $\Ocal(1)$. 
In contrast, strategies to sample node $u$ according to some non-uniform probability (\eg, a probability proportional to the node degree) are more costly, as they require to compute the cumulative sum of probabilities which is in $\Ocal(n^{(\ell)})$, if $n^{(\ell)}$ is the number of nodes.
%Moreover, the probability vector might need to be updated {\em after each edge sampling}, further increasing the computational complexity. 
Finally, we remove the refining steps when lifting the partition back to the input graph, as this will be taken care of in the next step with alignment and fusion of the different obtained partitions.
%Moreover, as we update the global coarsening matrix at each coarsening step $(H \leftarrow H H^{(\ell)})$, the lifting step simply reduces to 
% \begin{equation}
% \clustermatrix^{(0)} = H \clustermatrix^{(c)} \in \Pcal_{N,k} \,.
% \end{equation}
In Appendix~\ref{app:edge_sampling}, we detail other simple edge sampling rules that we investigated here but were unsatisfactory for the present work. However, it provides insight into the choices that led to our method. 
\addblocknew{
%All these choices lead to a fast coarsening algorithm. 
Finally, \Cref{prop:cplx_algo1} shows the complexity of the coarsening phase. % is linear with respect to the number of edges and logarithmic in the compression factor $\rho$.  
Its proof is deferred to Appendix~\ref{app:complexity}
\begin{proposition}
    The complexity of \Cref{alg:coarsening_global} is $\Ocal\left( (1+\log\rho) |E| \right)$, where $|E|$ is the number of non-zero coefficients in $\bigA$ (number of edges in the graph).
    \label{prop:cplx_algo1}
\end{proposition}}

\subsection{Structure preserving properties of the coarsening \label{subsec:phase_transition}}
%\el{works for me.}}

This section examines the properties and limitations of PASCO's coarsening on random graphs with community structures. To preserve community information, hypernodes must primarily consist of nodes from the same community, which requires collapsing intra-community edges. We analyze the conditions under which PASCO coarsening favors such edges, focusing on graphs generated by the Symmetric Stochastic Block Model (SSBM) defined below.

\begin{definition}[Symmetric Stochastic Block Model]
The SSBM is a random graph model with $N$ nodes divided into $k$ equal-sized communities. Each edge is present with probability $p_{in}$ if inside a community or $p_{out}$ if between communities, independently of all other edges. As in \cite{tremblay2016compressive}, we parametrize the model by $N$, $k$, the expected degree $d_N = d \log N$ and the intra-to-inter-community probability ratio\footnote{The use of $\alpha$ as a parameter is relevant for the experiments done afterwards, as we then keep the density fixed while varying the difficulty level $\alpha$ to recover the blocks.} $\alpha = p_{out}/p_{in}$. We refer to this model by $\SSBM(N,k,d, \alpha)$. 
\label{def:SSBM}
\end{definition}

% In the $\SSBM(N,k,d, \alpha)$, the parameter $d$ dictates the density of the graph, while the ratio $\alpha$ sets the difficulty of the clusters recovery\footnote{Although SSBM is usually parametrized with $N,k,p_{in},p_{out}$, we chose to use $d$ and $\alpha$, as in \cite{tremblay2016compressive}, because it will later relate to our experiments where we keep the density fixed and vary the difficulty level $\alpha$. We can always recover the probabilities of connection $p_{in},p_{out}$ from $N,k,\alpha,d$ by $p_{in} = {d \, k \log N}/{(N((1 + (k-1) \alpha ))}$ and $p_{out} = \alpha \, p_{in}
% $.}. 
% In particular, for $d>1$ and $\alpha>0$, the probability of having a connected graph tends to 1 with $N \to +\infty$. 
% Moreover, $\alpha > \alpha_c = (d - \sqrt{d}) / (d + (k-1) \sqrt{d})$ is the range where community detection is asymptotically impossible \cite{decelle2011asymptotic}. 

Consider an input graph drawn from an SSBM with $k$ communities, an edge probability inside communities of $p_{in}$, and an edge probability between communities of $p_{out}$. In PASCO, edges to collapse are obtained by first drawing some node $u$ and taking a random neighbor. 
In expectation, $u$ has $n p_{in} / k$ neighbors from its community and $ n (k-1) p_{out} / k$ neighbors from other communities. So when $ p_{in} > (k-1) p_{out} $, $v$ is more likely to be from the same community as $u$. More generally, under this condition, we expect the coarsening procedure of PASCO to collapse more inside-community edges than between-community edges. Therefore, we conjecture that PASCO conserves the community structure of a graph drawn from a $\SSBM(N,k, d, \alpha)$ as long as $\alpha = p_{out}/p_{in} < 1 / (k-1)$.
% \begin{equation}
% \alpha = \frac{p_{out}}{p_{in}} < \frac{1}{k-1} \,.
% \label{eq:pasco_phaseTrans}
% \end{equation}
\add{This remains a conjecture, not yet supported by rigorous proof. However,} experiments are providing empirical evidence that this phase transition correlates with PASCO's performance as described in our experiments in  \Cref{subsec:expe_coarsening} and \Cref{subsec:expe_SBM}.

\section{Alignment and fusion\label{sec:align_details}}

% \el{remove linreg and manytoone and put them in the appendix}\\
% \el{Eventually, quickly mention that we investigated quad-ot but is was not satisfactory (slower and similar perfs).} \\
% \el{Make sure to be clear on what does already exist and what is new (almost nothing excepct quad-ot)}

The last step of PASCO is to align and combine the partitions obtained from various coarsened graphs. 
We advocate the use of an OT-based approach to align the partitions, as initially proposed in \cite{li2019optimal}. 
The key idea is to define a notion of distance between partitions by using the Wasserstein distance $\OT_2$. 
Precisely, one can represent a partition matrices $\clustermatrix \in \mathcal{P}_{N, k}$ as a discrete probability distribution. Writing $\clustermatrix = (p_1, \cdots, p_k)$ where $p_j$ is the $j$-th column of $\clustermatrix$, the $j$-th cluster can be represented by the vector $p_j \in \{0,1\}^{N}$. The discrete probability distribution in $\R^{N}$ associated \add{with} $\clustermatrix$ is then given by $\mu_{\clustermatrix}= \frac{1}{k} \sum_{j=1}^{k} \delta_{p_{j}}$ where $\delta$ is the Dirac mass. %This discrete probability measure has a particular structure: the $p_j$ are vectors with coefficients in $\{0,1\}$ and they all have disjoint supports ($\forall j \neq j', \langle p_j, p_{j'} \rangle = 0$).

For two partitions  $\clustmatbar \in \mathcal{P}_{N, \kbar}, \clustermatrix \in \mathcal{P}_{N, k}$, we can compare them by comparing their associated probability measures $\mu_{\clustmatbar}, \mu_{\clustermatrix}$ through the Wasserstein distance
\begin{equation}
\label{eq:ot_problem}
\begin{split}
\OT_2^2(\mu_{\clustermatrix}, \mu_{\clustmatbar}) &= \min_{Q \in \Qcal_{\ot}(k, \kbar)} \ \sum_{i,j=1}^{k, \kbar} \|p_{i} - \overline{p
}_{j}\|_2^2 Q_{i,j}\,.
\end{split}
\end{equation}
In \Cref{eq:ot_problem} the set $\Qcal_{\ot}(k, \kbar) \subseteq \R_+^{k \times \kbar}$ denotes the collection of all {\em coupling matrices}, {\em i.e.} matrices
$Q \in \R_{+}^{k \times \kbar}$ that satisfy the marginal constraints: $Q\one_{\kbar} = \frac{1}{k} \one_{k}$ and $\ Q^\top\one_{k} = \frac{1}{\kbar} \one_{\kbar}$.
Intuitively, the element $Q_{i,j} \in [0,1]$ represents the amount of probability mass shifted from the $i$-th cluster of $\clustermatrix$ to the $j$-th cluster of $\clustmatbar$. 
%The marginal constraints ensures that the total mass of the clusters in partition $\clustermatrix$ is send to the partition $\clustermatrix'$ and that each cluster in $\clustermatrix'$ must receive a mass $1/k$. 
This coupling can be used to align the clusters of the two partitions: in the special case where $k = \kbar$, an optimal solution is given by the permutation that best realigns the clusters \cite{peyre2019computational}. We point out that solving \eqref{eq:ot_problem} is done through a linear program that can be computed with standard solvers \cite{flamary2021pot} with a worst-case complexity $\Ocal(K^3 (\log K)^2)$ where $K = \max\{k, \kbar\}$.
We rely on this distance to achieve a consensus among the different partitions by solving an OT barycenter problem: we fix a number of desired clusters $\kbar$ and look for the partition matrix $\clustmatbar \in \mathcal{P}_{N, \kbar}$ that minimizes \eqref{eq:ot_barycenter}.

As described in \cite{cuturi2014fast}, this barycenter problem can be tackled by alternating between solving $R$ problems of OT and updating the reference $\clustmatbar$. 
As described in \Cref{lemma:fusion_update}, the reference update can be obtained in closed-form by a simple majority vote as follows.
\begin{equation}
\label{eq:fusion_update_ref}
\forall i \in \integ{N}, [\clustmatbar]_{ij} \leftarrow \begin{cases} 1 & j \in \underset{p \in \integ{\kbar}}{\operatorname{argmax}} \ [\sum_{r=1}^{R} P_r Q_r]_{ip} \\ 0 & \text{otherwise}  \end{cases}\,,
\end{equation}
where $Q_1, \cdots, Q_R$ are the optimal coupling matrices obtained in the previous step when solving the individual OT problems between the previous reference partition and the $\clustermatrix_1, \cdots, \clustermatrix_R$.
The algorithm for alignment and fusion is sketched in \Cref{alg:alginment_fusion}.
This OT-based alignment + fusion algorithm  %described in \Cref{alg:alginment_fusion} 
requires a choice for the target number of clusters $\kbar$ and an initial reference $\clustmatbar$. We use the following heuristic in practice: If all the $k_r$ are equal, we choose $\kbar = k_{1}$ and initialize the reference partition with $\clustmatbar = P_{1}$, otherwise we choose $\kbar$ as the $k_r$ closest to the median number of clusters across the partitions in the dataset $\operatorname{median}(k_1, \cdots, k_R)$, and we initialize $\clustmatbar$ as the corresponding partition. 

\begin{wrapfigure}{R}{0.6\linewidth}
\vspace{-3em}
	\begin{minipage}{\linewidth}
	\begin{algorithm}[H]
	\caption{\label{alg:alginment_fusion}Alignment \& Fusion algorithm}
	\begin{algorithmic}
	\addblock{
	 \State Partitions ${\clustermatrix}_1, \cdots, {\clustermatrix}_R$, number of clusters $\kbar$ and initial reference $\clustmatbar$.
	  \While{$\clustmatbar$ has not converged}
	  \For{$r \in \integ{R}$}\Comment{Alignment step}
	  \State Find $Q_r$ so that $\clustermatrix_r Q_r$ is aligned on $\clustmatbar$ (\ie, solve the OT problem \eqref{eq:ot_problem} in \Cref{sec:align_details}) 
	  \EndFor
	  \State{Update $\clustmatbar$ using majority vote on $(\clustermatrix_r Q_r)_{r \in \integ{R}}$, as in \eqref{eq:fusion_update_ref} in \Cref{sec:align_details}.} 
	  \Comment{Fusion step}
	  \EndWhile \\
	  \Return{$\clustmatbar \in \{0,1\}^{N \times \kbar}$}
	  }
	\end{algorithmic}
	\end{algorithm}
	\end{minipage}
\end{wrapfigure}

\textbf{Complexity analysis:} 
Let $K = \max\{\kbar, k_1, \cdots, k_R\}$ then the algorithm runs in $\Ocal(n_{\iter} R (N K^2 + K^3 \log(K)^2))$ where $n_{\iter}$ is the number of iterations required for $\clustmatbar$ to converge.
Overall, the algorithm scales linearly in $N$ and has roughly a cubic complexity \textit{w.r.t.} the number of clusters, which is often small compared to the number of nodes. In practice, we observe that $n_{\iter}$ is small (on the order of 10). 

\textbf{Other methods for alignment + fusion:} As an alternative to the described OT approach, we also investigated the so-called \emph{linear-regression-based} and the \emph{many-to-one} methods \cite{ayad2010voting}, as well as a slightly different variant of OT based on quadratic-penalized OT \cite{blondel2018smooth}. 
These methods also solve a barycenter problem but for other notions of distance between partitions. 
In practice, we find that the standard OT-based method performs better for our application (see Section \ref{sec:expe} and especially Figures~\ref{fig:illustration_fusion} and \ref{subfig:influence_align}). 
The presentation of these alternative methods is deferred to \Cref{sec:consensus}.

\FloatBarrier
\section{Experiments \label{sec:expe}}

We conducted experiments to evaluate PASCO. Section~\ref{subsec:expe_coarsening} shows that the coarsening step preserves well graph spectral properties. Section~\ref{subsec:expe_SBM} and Section~\ref{subsec:expe_real} evaluate PASCO on SSBM and real graphs respectively. See \Cref{app:implem_details} for details about computing resources. Section~\ref{subsec:expe_fusion} studies the alignment and fusion phases. The code for PASCO and the experiments is available at \url{https://github.com/elasalle/PASCO} and at \cite{lasalle:hal-05086352v1}.

\subsection{Coarsening: conservation of graph spectral properties}
\label{subsec:expe_coarsening}

\addnew{The performance of our coarsening phase is assessed through Loukas's analytical framework \cite{loukas2019graph}}, \addnew{where} coarsening techniques \addnew{are evaluated} based on the spectral properties of the graph Laplacian. Loukas introduced the \emph{Restricted Spectral Approximation} (RSA) to quantify how well the projection matrix $\Pi = H H^+$ approximates the identity on $\mathcal{U}_k$, the subspace spanned by the $k$ eigenvectors of $L$ associated \add{with} the {\em smallest} eigenvalues. RSA is defined as the smallest $\varepsilon$ such that $\|x - \Pi x\|_L \leq \varepsilon \|x\|_L$ for all $x \in \mathcal{U}_k$. Here, $L$ is the combinatorial Laplacian, $H$ is the binary coarsening matrix, and $H^+$ is its pseudo-inverse.

\begin{figure}
    \centering
    \includegraphics[width=\linewidth]{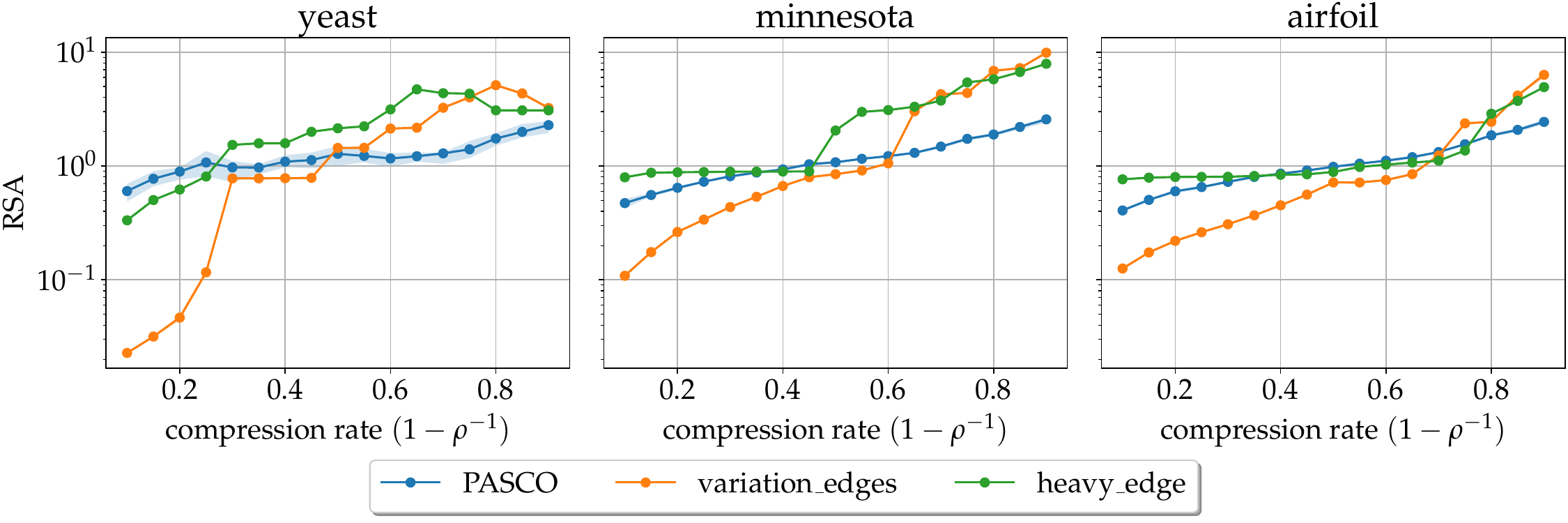} % Replace with your image path
    \caption{We represent the RSA (the smaller, the better) of various coarsening schemes (including PASCO) as a function of the compression rate (the higher, the coarser the obtained graph). Shaded areas represent 0.2 upper- and lower-quantiles. 
    %already said: Top row: we reproduce a part of the experiment of \cite[Figure 2]{loukas2019graph} with the same real graphs and added PASCO. Bottom row: same experiment but with random graphs drawn from $\SSBM(1000, k, 2, \alpha)$ for $(k=10, \alpha=1/k)$, $(k=100, \alpha=1/k)$ and $(k=10, \alpha=1/(2k))$.
    }
    \label{fig:RSA}
\end{figure}

% \begin{table}
%     \centering
%     \begin{tabular}{c|c|c|c}
% & PASCO & \texttt{variation\_edges} & \texttt{heavy\_edge}  \\ \hline
% yeast & 5.e-3 (3.0e-4) & 4.3e-1 (4.2e-2) & 2.5e-1 (3.1e-2)  \\ \hline
% minnesota & 1.8e-2 (3.8e-3) & 4.2e-1 (6.5e-2) & 4.1e-1 (8.1e-2)  \\ \hline
% airfoil & 2.1e-2 (3.e-3) & 7.8e-1 (4.5e-2) & 5.3e-1 (2.9e-2)  \\ \hline
% SSBM(10, 0.1) & 7.5e-3 (3.9e-4) & 8.e-1 (3.1e-2) & 2.0e-1 (1.7e-2)  \\ \hline
% SSBM(100, 0.01) & 7.7e-3 (5.6e-4) & 6.8e-1 (6.9e-2) & 1.8e-1 (1.8e-2)  \\ \hline
% SSBM(100, 0.005) & 7.5e-3 (4.6e-4) & 5.8e-1 (3.5e-2) & 1.7e-1 (1.6e-2)  \\ \hline
%     \end{tabular}
%     \caption{The mean computational time (in seconds) of the coarsening methods, for a fixed compression rate $r=0.75$. The standard deviation is given in parenthesis. }
%     \label{tab:RSA_one_time}
% \end{table}

\textbf{Experimental setting:} In these experiments, we compare our method in terms of RSA and computational time with existing \addnew{coarsening} methods\addnew{: \texttt{heavy\_edge} \cite{karypis1998fast,loukas2018spectrally} and \texttt{variation\_edges} \cite{loukas2019graph}} . \addnew{The former iteratively collapses the edge with the largest weight. It is fast and simple and represents the basis of more evolved approaches. The latter essentially applies \texttt{heavy\_edge} newly defined weights such that the coarsening procedure optimizes the RSA. Therefore, it is interesting to compare our coarsening step with \texttt{variation\_edges} in terms of RSA.} 

Figure~\ref{fig:RSA} shows the results for our proposed method (PASCO) \addnew{and the concurrent methods.} %as well as previously existing methods, namely, the local variation method based on edges proposed by Loukas \cite{loukas2019graph} (\texttt{variation\_edges}) and the heavy edge collapsing method (\texttt{heavy\_edge}) \cite{karypis1998fast,loukas2018spectrally}).  
We display the \texttt{RSA} values with respect to the \emph{compression ratio} $1-\rho^{-1}$. Recall that $\rho$ is the \emph{compression factor} defining how much the initial graph is coarsened ($n = \lfloor N/\rho \rfloor$). Assuming that $N/\rho$ is an integer, the \emph{compression ratio} corresponds to the ratio $(N-n)/N = 1-\rho^{-1}$. \add{As the coarsening step of PASCO is random, we average the performance over 10 runs. The other approaches are not repeated as they are deterministic.}
\add{
The experiments are performed on the real graphs used in \cite{loukas2019graph} as well as on random SSBM graphs. The real graphs include \texttt{yeast} a protein-protein interaction network, \texttt{airfoil} a mesh for airflow simulation, and \texttt{minnesota} a road network (see \Cref{tab:graph_caracteristics} for some graph characteristics).
\Cref{fig:RSA} shows the results for the real graphs while \Cref{fig:RSA_SBM} shows the SSBM ones. We provide the computational times of each coarsening method for both dataset in \Cref{fig:RSA_timings}.}

\textbf{Results:} While \texttt{variation\_edges} and to some extent \texttt{heavy\_edge} generally yield the best RSA at \emph{small compression rates}, we emphasize that PASCO was not specifically tailored to preserve such spectral properties unlike \texttt{variation\_edges} which is designed to optimize the RSA.
Moreover, we are rather interested in \emph{high compression rates}, as significant clustering computation time gains are to be expected. 
In this high compression regime, despite the fact that \texttt{variation\_edges} was tailored to optimize the RSA, PASCO is still competitive, especially on real graphs where it outperforms both \texttt{variation\_edges} and \texttt{heavy\_edge}.
Moreover, PASCO proves to be much faster than both other methods: by more than a factor 10 for \texttt{heavy\_edge} and around a factor 100 for \texttt{variation\_edges}  (see \Cref{fig:RSA_timings}).
These experiments demonstrate that, although the coarsening step of PASCO is primarily designed for computational efficiency, it also preserves the structures of the initial graph as effectively as, or sometimes even better than, traditional coarsening methods. They also demonstrate that PASCO is much faster than the other coarsening algorithms tested, as expected.

\subsection{Synthetic graph experiment and parameter analysis.}
\label{subsec:expe_SBM}

We explore the performance of PASCO (coupled with spectral clustering \add{(SC)} to cluster the coarsened graphs) on an initial graph generated using the symmetric stochastic block model. 

\textbf{Quality of the output partition:} To study the influence of the hyperparameters of PASCO, we conduct an experiment on synthetic graph data from the Symmetric Stochastic Block Model $\SSBM(N, k,d, \alpha)$, see \Cref{def:SSBM}. We take $N=10^4$ and $k=20$ and set $p_{in}$ and $p_{out}$ such that the average degree is $d_N = d \log(N)$ with $d = 3/2$. We vary the fraction $\alpha = p_{out}/p_{in}$ from 0 (excluded) to $\alpha_{\sup} = \frac{4}{3} \frac{1}{k-1}$. This range includes both the phase transition threshold of PASCO, of value ${1}/{(k-1)}$ (as conjectured in \Cref{subsec:phase_transition}) and the threshold of exact recovery for spectral clustering, \ie, $\alpha_c = (d - \sqrt{d}) / (d + (k-1) \sqrt{d})$.

% Experiment description

For each set of SSBM parameters, we draw 10 graphs, ensuring their connectedness with rejection sampling. For each graph, we compute the performance of PASCO with spectral clustering measured by the AMI. We study the influence of each parameter: the compression factor $\rho$ (such that $n = \lfloor N/\rho \rfloor$), the number of coarsened graphs $R$, and the method used for alignment $\texttt{align\_method}$. Each parameter varies as follows, while the others are kept constant to a default value (written here in bold): $\rho \in \{1,3,5,\mathbf{10},15,20\}$, $R \in \{1,3,5,\mathbf{10},15,20 \}$, $\texttt{align\_method} \in \{ \texttt{lin\_reg}, \texttt{many\_to\_one}, \textbf{\texttt{ot}}\}$. Performance is averaged over the 10 realizations and displayed in Figure~\ref{fig:param_influence}. 

\begin{figure}  
\centering
\begin{subfigure}{0.45\linewidth}  
\includegraphics[width = \textwidth]{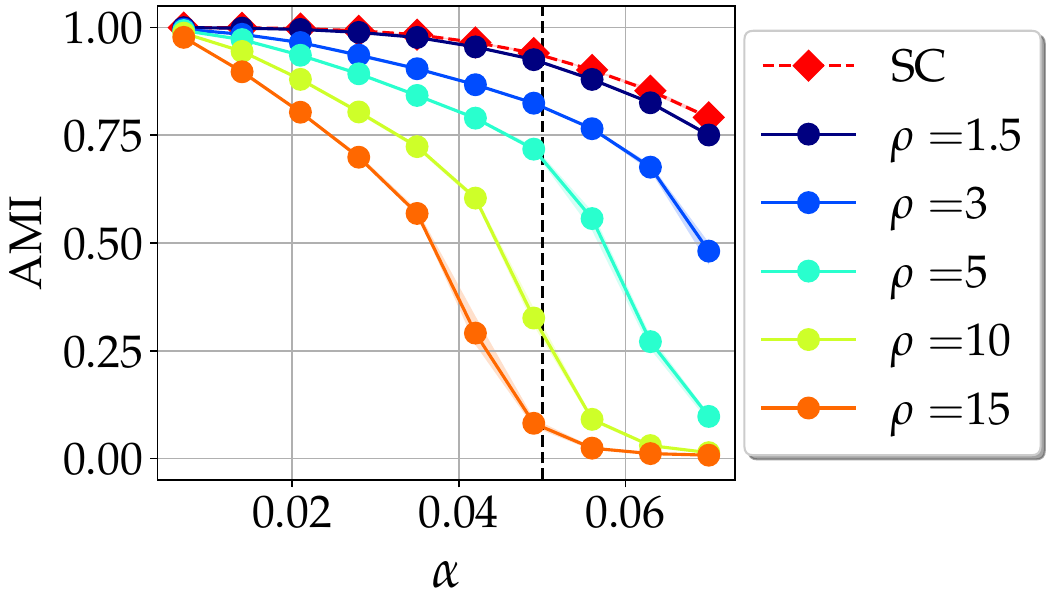}
\caption{Influence of $\rho$ ($R=10$).}
\label{subfig:influence_rho}
\end{subfigure} 
\begin{subfigure}{0.45\linewidth}  
\includegraphics[width = \textwidth]{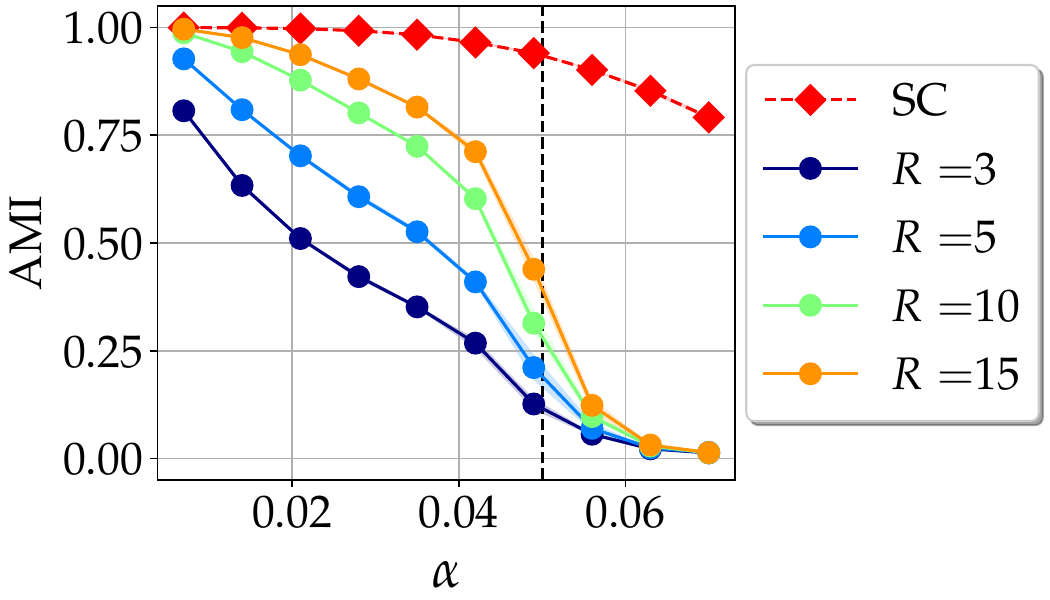}
\caption{Influence of $R$ ($\rho=10$).}
\label{subfig:influence_n_tables}
\end{subfigure} 
\caption{Parameters influence on PASCO. The AMI score is average over the 10 runs and shaded areas represent 0.2 upper- and lower-quantiles. Dashed lines correspond to the PASCO threshold.}
\label{fig:param_influence}
\end{figure}    

% Comment on the performance plots
The compression factor $\rho$ determines how small the coarsened graphs are. Hence, the larger $\rho$, the harder it is to retrieve the communities of the input graphs (see Figure~\ref{subfig:influence_rho}). Moreover, it is even more difficult to recover the communities when they are not much denser than the rest of the graph (large $\alpha$). However, better performance is achieved with more coarsened graphs (bigger $R$); see Figure~\ref{subfig:influence_n_tables}. The rise in AMI with $R$ confirms that the alignment and fusion process is effectively able to combine the noisy information contained in each clustering. We also observe a change in the behavior of PASCO after the conjectured threshold ($\alpha > 1/(k-1)$, vertical dashed line). 
%At this point, the coarsening is more likely to put vertices of different communities in the same super node, hence losing the community structure. Therefore, the quality of the individual partitions is drastically reduced, and the fusion part is no longer able to recover a good partition. 
In the Appendix, Figure~\ref{subfig:influence_align} indicates that all alignment methods perform similarly, excepted when close to PASCO's conjectured phase transition as then $\texttt{ot}$ outperforms the other methods. 

\textbf{Study of the computational time:} In this experiment, we study the gains in computational time due to PASCO (coupled with spectral clustering (SC) to cluster the coarsened graphs), compared to plain SC.
As highlighted in the previous experiment, larger compression factors $\rho$ should be compensated by larger numbers of coarsened graphs $R$ to ensure good partition quality. In the following, we arbitrarily decide to fix the number of coarsened graphs equal to the compression factor, \ie, $R = \rho$ and study the impact of $\rho$ on computational time. The results are presented in Figures~\ref{fig:pasco_timing}. The SSBM parameters are set to $N=10^5$, $d=3/2$, and $\alpha = 1/(2(k-1))$ and $k$ varies in $\{20, 100, 1000\}$. For each set of SSBM parameters, we draw 10 graphs, ensuring their connectedness thanks to rejection sampling. For each graph, we compute the computational time of PASCO for values of $\rho$ in $\{3,5,10,15\}$. 

% \begin{definition}[Ideal time]
% Here is how we compute the ideal computational time. Let $\texttt{T\_co}$ and $\texttt{T\_fu}$ denote the coarsening and fusion computational times, respectively. Moreover, denote by $\texttt{T\_cl}$ the time required to compute all $t$ clustering of the coarsened graphs sequentially (i.e., without parallelization). Then, the ideal computational time is defined by 
% \[\texttt{T\_ideal} \stackrel{\Delta}{=} \texttt{T\_co} + \frac{\texttt{T\_cl}}{t} + \texttt{T\_fu} \,.\] 
% We use this \emph{ideal time} when PASCO is combined with a clustering algorithm whose implementation already involves some parallelization (e.g., SC). \el{Je vous laisse préciser ces histoires de coeurs, threads, etc. J'y comprends rien moi!}
% \end{definition} 

\begin{figure}
\centering
\begin{subfigure}{0.45\linewidth}
\includegraphics[width=\linewidth]{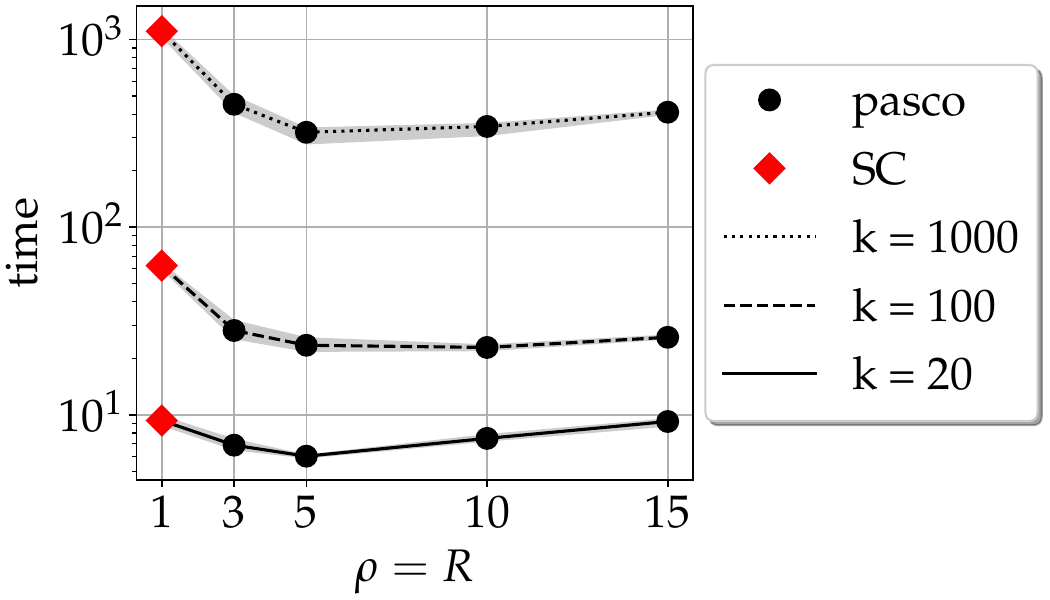}
\caption{Computational time.}
\label{subfig:pasco_timing}
\end{subfigure}
\begin{subfigure}{0.45\linewidth}
\includegraphics[width=\linewidth]{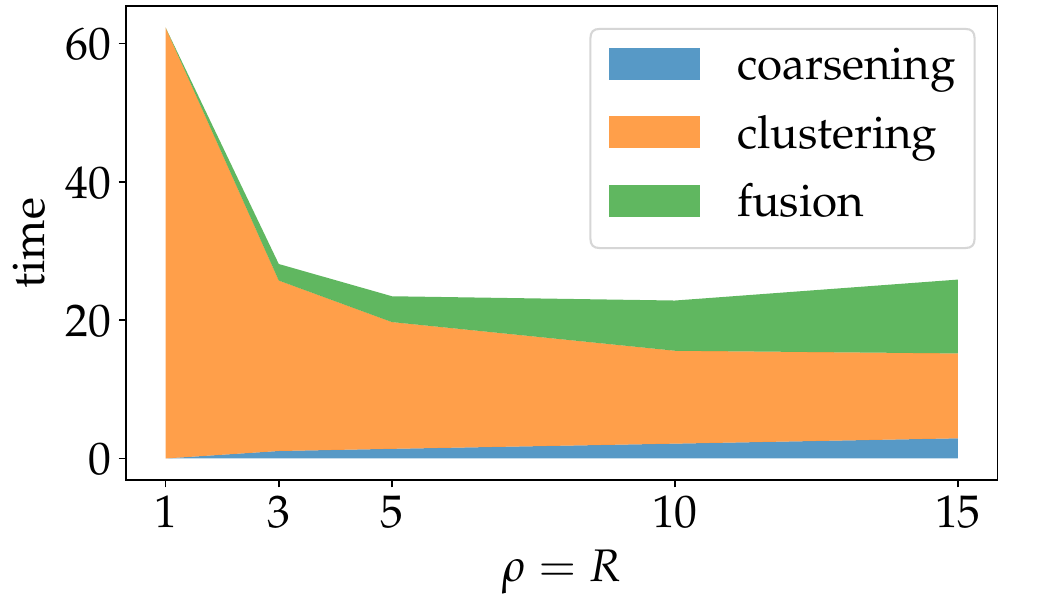}
\caption{Time details $(k=100)$.}
\label{subfig:stacked_times}
\end{subfigure}
\caption{Influence of $\rho$ and $R$ on PASCO computational time, letting $R=\rho$. The number of communities varies in $\{20, 100, 1000\}$. Timings are averaged over 10 runs and shaded areas represent 0.2 upper- and lower-quantiles.}
\label{fig:pasco_timing}
\end{figure}

\textbf{}
% Comment on the time plots 
Figure~\ref{subfig:pasco_timing} shows that using $\rho=R>1$ accelerates the overall computation compared to plain SC, which corresponds to $\rho=R=1$. For this experiment, an empirical optimum is found around $\rho=R=5$. The overall speedup is greater when the number of communities $k$ is larger. Indeed, in this case, the spectral decomposition in SC is really computationally demanding, and performing it on smaller coarsened graphs leads to a significant improvement in the computational time. Moreover, even for large $\rho$, the coarsening and fusion part still amount to only a small proportion of the computational effort, as illustrated in Figure~\ref{subfig:stacked_times}.  
Figure~\ref{subfig:quality_for_timing} attests that the significant speedup of Figure~\ref{subfig:pasco_timing} does not come at the cost of poor quality of the output partitions.

\subsection{Real graphs experiment \label{subsec:expe_real}}

% Dataset description
To further validate the ability of PASCO to improve the performance of clustering algorithms, we conduct experiments on real datasets. We consider three large graphs (\texttt{arxiv}, \texttt{mag}, and \texttt{products}) come from \texttt{Open Graph Benchmark}~\cite{hu2020ogb}.  
The experiments are run on the largest connected component of each graph. Their characteristics are in Table \ref{tab:real_datasets}. Some datasets have features associated with the nodes, but they are ignored in these experiments as the study is limited to non-attributed graph clustering. \add{Only the graph structure is used for clustering.}

\textbf{Clustering Algorithms:} We below list the clustering algorithms that we use within the PASCO pipeline, and we provide details on the graph characteristics they try to optimize. The \emph{Spectral Clustering} algorithm (SC) exploits eigenvectors of the Laplacian matrix. It minimizes a notion of \emph{generalized normalized cut} (gnCut) \cite{dhillon2004kernel}. %, see \Cref{def:gnCut}. 
A modified version of SC was proposed by \cite{tremblay2016compressive} to accelerate and reduce memory print and is called \emph{Compressive Spectral Clustering} (CSC). 
The classical \emph{Louvain} method \cite{blondel2008fast} and its modern variant \emph{Leiden} \cite{traag2019louvain} are both modularity maximization methods. %, see \Cref{def:modularity}.
We recall that modularity measures the partition quality by comparing inside/outside community densities \cite{newman2004finding}.
Finally, we also include clustering algorithms that use the Minimization of Description Length (MDL) to either maximize the likelihood of the stochastic block model~\cite{peixoto2014efficient}, or the compression of the graph into clusters as per the \emph{infomap} method~\cite{infomap:2011}. These algorithms rely on the \emph{Description Length} (DL) \cite{Rissanen:2007} that quantifies the amount of information required to describe the parameters of an SBM adjusted to the observed graph. %A small Description Length corresponds to a simple model that explains the observed graph well. 
In \Cref{app:scores}, we recall the definition of the scores mentioned above: gnCut, modularity, DL, as well as the AMI.

\begin{wrapfigure}{r}{0.5\linewidth}
    \vspace{-1.5em}
    \centering
    \includegraphics[width=0.8\linewidth]{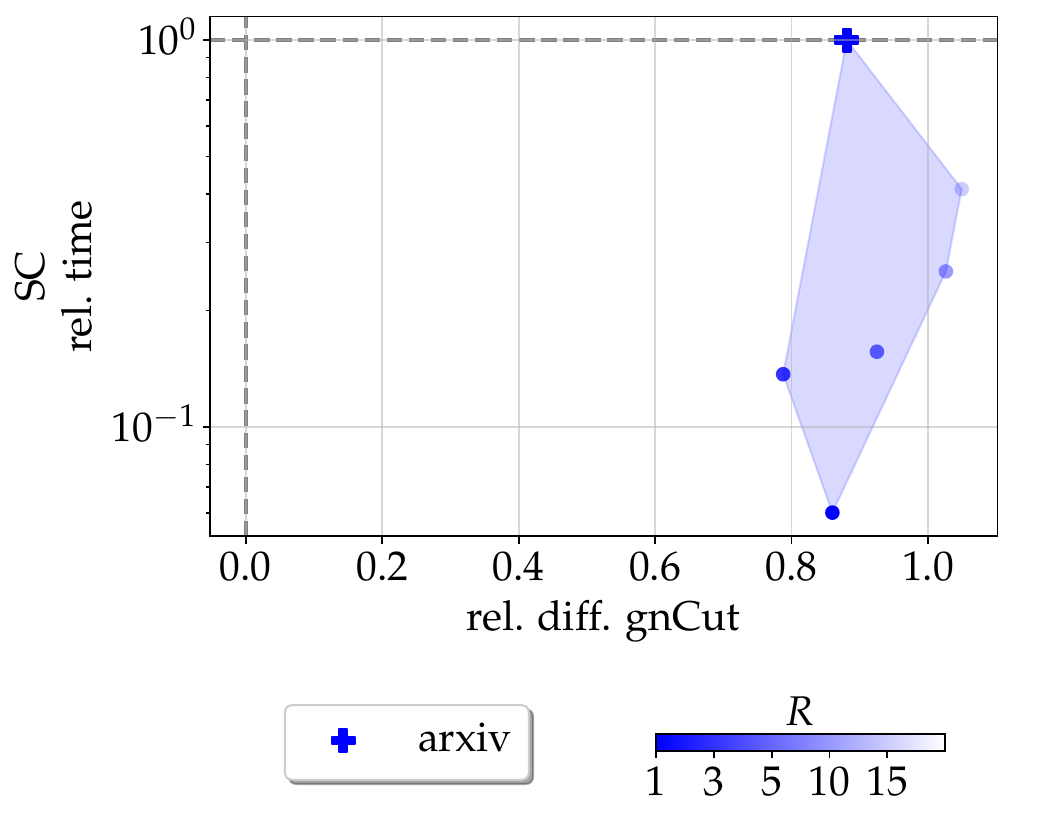}
    \caption{\add{An example to help read Figure~\ref{fig:real_ytime}. Each plot represents the performance of PASCO associated with a given clustering method (here SC). Specific markers (here a cross) represent the standalone clustering method’s performance, while disk markers show PASCO’s. Color (here blue) indicate the dataset (here \texttt{arxiv}). 
    Performance is evaluated by computational time (y-axis, relative to standalone method, hence specific markers at $y=1$) and clustering quality (x-axis, measured by AMI with the true partition or relative quality score difference; here, generalized normalized cut).
    Disk transparency varies with the number of coarsened graphs $R$ (see color bar). Shaded area represents the convex hull of points associated with a dataset, ideally trending downward (speedup) and either rightward for AMI or closer to zero for relative quality differences.}}
    \label{fig:real_example}
    \vspace{-2em}
\end{wrapfigure}

\textbf{Experimental setting:} 
We evaluate PASCO's impact on computational time and partition quality by comparing clustering methods with and without PASCO. Graphs are clustered using SC, CSC, Louvain, Leiden, MDL, and Infomap, then re-clustered by applying PASCO combined with the same clustering algorithms. Except for CSC, which we re-implemented in Python due to it being only available in Matlab, standard implementations were used. The compression factor $\rho$ is fixed at 10, and coarsening repetitions $R$ vary in $\{1, 3, 5, 10, 15\}$. Partition quality is assessed using AMI for agreement with ground truth and intrinsic scores like modularity, gnCut, and DL.
For these scores, the relative difference $(s_\textnormal{est} - s_\textnormal{true}) / s_\textnormal{true}$ is reported,
% \begin{equation}
% (s_\textnormal{est} - s_\textnormal{true}) / s_\textnormal{true} \,,
% \label{eq:rel_diff}
% \end{equation}
where $s_\textnormal{est}$ and $s_\textnormal{true}$ are the scores of the, respectively, estimated and ground-truth partitions. 
We also compute the total computational time for each method and report the ratio between PASCO running time and the one of the standalone clustering method. See the paragraph below for some precision on how computational time is measured. 
The results are in Figure~\ref{fig:real_ytime}. 
% To guide the reader, we extract first a simpler example in \Cref{fig:real_example}. 
\add{We explain how to read it on a simpler example in \Cref{fig:real_example}.}
For completeness, \addnew{tables of \Cref{app:real_data_tables} provide the numerical results}.

\textbf{Measuring Computational time:} 
To measure computational time, we record the duration from start to finish. Since PASCO relies on parallelization, we run each clustering method on a single core, whether used alone or with PASCO. This ensures that PASCO's parallelization does not conflict with that of the clustering methods, allowing us to evaluate PASCO's speedup effect. Consequently, clustering methods optimized with parallelization will have diminished performance, so comparisons between algorithms based on computational time are not meaningful. However, we can effectively analyze PASCO's impact on individual clustering methods. Extending PASCO to support parallelized clustering across multiple cores or machines is beyond the scope of this article.

\textbf{Results:} 
The experiments show that \addnew{the PASCO clustering pipeline} improves runtime or clustering quality for most clustering methods and graphs under study, sometimes achieving both. With SC, CSC, or Infomap, PASCO significantly reduces runtime in most cases, often by a factor of 10 or more. For Leiden, runtime improvements are smaller and are mainly seen with fewer coarsening repetitions ($R < 5$). With MDL, the effect is lighter and more data-dependent. No runtime improvement is observed with Louvain, as its fast clustering is offset by the additional time for PASCO's coarsening and fusion steps. 

%--  PB: this paragraph was repeated in the newt one --
% PASCO maintains the clustering quality of the methods used within its pipeline, validating the adequacy of the chosen heuristics. For AMI, which measures agreement with the ground truth, PASCO slightly improves results for SC and CSC, while other methods show similar performance. Modularity and Description Length scores tend to be more ``regularized," aligning closer to the ground truth partition scores. The gnCut criterion is generally preserved, with one exception (MDL on the \texttt{mag} graph). This preservation may relate to the coarsening step’s spectral property retention, as discussed in \Cref{subsec:expe_coarsening}, although further theoretical analysis is needed.

\addnew{The PASCO clustering pipeline} achieves a similar quality in clustering as the methods it uses. This is a notable result, given that it was initially designed to reduce computational costs. %and this property on the preservation of partition quality is only supported by heuristics.
For AMI, which compares partitions with ground truth, \addnew{the PASCO clustering pipeline} improves the results for SC and CSC, while other methods show no significant change. Modularity and Description Length scores tend to be more ``regularized", with values closer to ground truth scores. The gnCut criterion is generally preserved, except for one case (MDL on \texttt{mag}). This preservation may be related to the ability of the coarsening step to preserve spectral properties (\Cref{subsec:expe_coarsening}). %, although further theoretical study would be needed to better understand the reason behind that.

\addnew{In parts of the experiments, the PASCO clustering pipeline was combined with multilevel clustering algorithms (\eg, Louvain, Leiden, Infomap), while PASCO also fits in this category. In these cases, doubling the multilevel phases (one in PASCO and one in the clustering method) only yields mitigated computational gains (except Infomap). Indeed, because of their multilevel nature, the standalone methods are already fast and hard to accelerate. However, we observe improvements in the AMI and the clustering criterion (closer to the value of the true partition) when using PASCO, consolidating the fact that computing multiple randomized coarsenings yields a more robust final partition.}

%\RG{requires some comment here: was designed explicitly to reduce costs, with heuristics to "hope for" preserved quality; it is good experimental news \ldots }. 

% \begin{figure}
% 	\centering
% 	\begin{minipage}{0.45\textwidth}
% 		\centering
% 		\includegraphics[width=\linewidth]{../expes/figures/real_graphs/simple_perf_vs_ideal_time_arxiv_leiden_louvain_lobpcg_CSC_MDL.png}
% 	\end{minipage}\hfill
% 	\begin{minipage}{0.45\textwidth}
% 		\centering
% 		\includegraphics[width=\linewidth]{../expes/figures/real_graphs/simple_perf_vs_time_arxiv_leiden_louvain_lobpcg_CSC_MDL_paraClust}
% 	\end{minipage}
% 	\caption{Performances of different methods with PASCO for the arxiv dataset. Left: ideal time. Right: effective computation time.}
% 	\label{fig:real_datasets}
% \end{figure}
%\begin{wrapfigure}[12]{R}{0.5\textwidth}

%\end{wrapfigure}

\begin{figure}
    \centering
    \includegraphics[width = \linewidth]{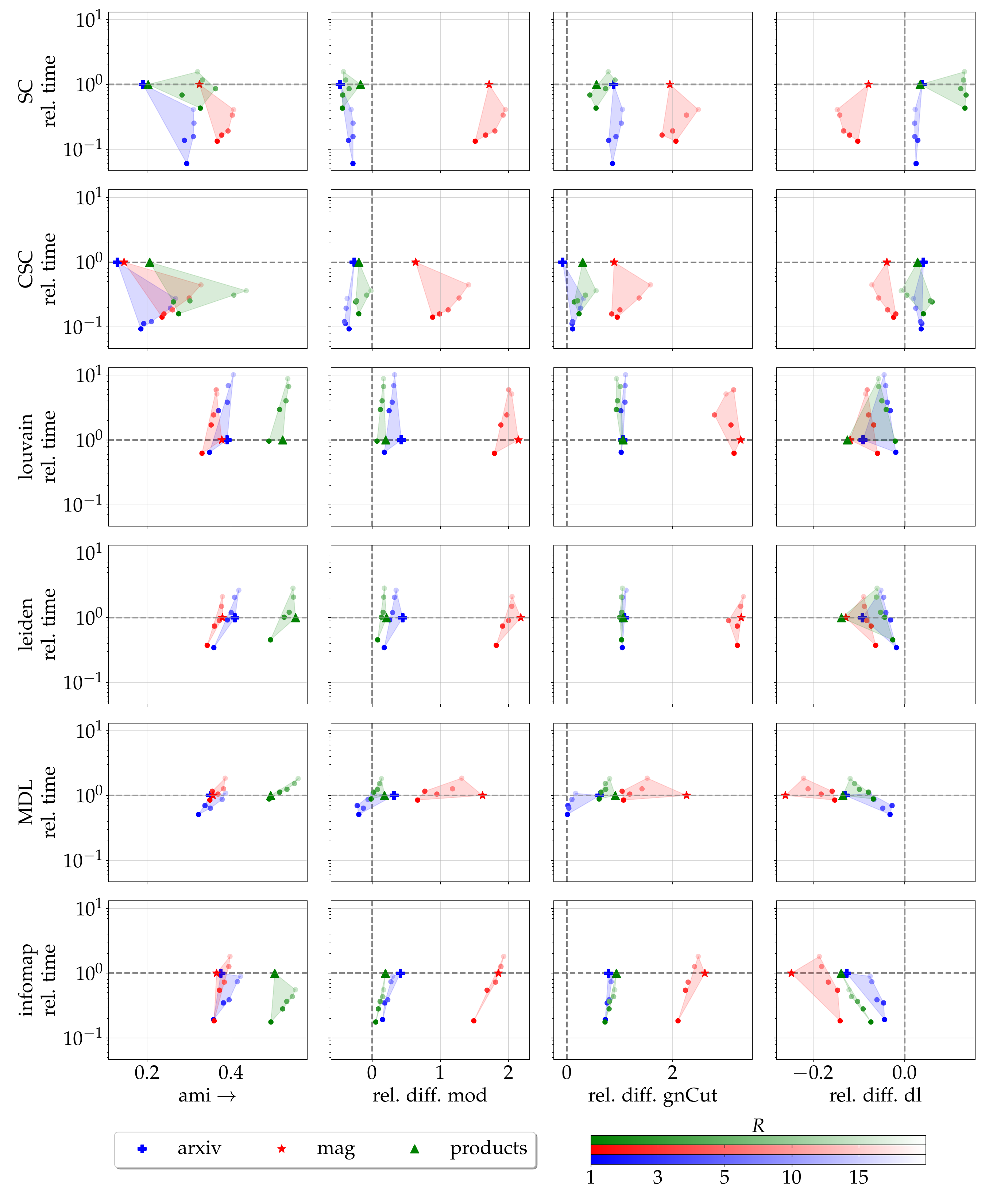}
    \caption{Results of the real graphs experiment \add{(See \Cref{fig:real_example} for an isolated example helping to read each of the sub-figures)}. Rows correspond to clustering methods, while columns correspond to clustering quality measures. On $y$-axes,  we represent the computational time relative to the one of the standalone method. On the $x$-axis we display either the AMI with the true partition or the relative quality score difference. This experiment is performed with different graphs (colors) and for different numbers of repetitions of the coarsening $R$ (transparency). The reduction factor is $\rho = 10$.}
    \label{fig:real_ytime}
\end{figure}

%\FloatBarrier

% !TEX root = ../main.tex

\section{Conclusion \label{sec:conclusion}}

We introduced \addnew{the PASCO clustering pipeline}, a novel approach to accelerate graph clustering algorithms through a coarsening-based strategy. \add{It is designed to be applied to a variety of clustering methods} and is built around three main steps: reducing the graph via a \addnew{new} fast and empirically structure-preserving random coarsening process, running clustering algorithms in parallel on the coarsened graphs, and combining the resulting partitions using an optimal transport-based fusion technique. 
The experimental results demonstrate the efficiency of PASCO: significantly reducing computational time while maintaining, or even enhancing, the quality of the resulting partitions, as shown on both synthetic and real-world graph datasets. 
As PASCO is a modular framework, it can be seamlessly integrated with various clustering algorithms, making it a versatile tool suitable for large-scale graph clustering challenges. \addnew{It will be particularly interesting when the number of expected communities is large (\eg, trading networks, gene networks or social networks), as the PASCO clustering pipeline produces the best computational gains in this setting.}

\add{ Going further, it would be interesting to explore the use of PASCO on more complex computing architectures. The goal would be to fully exploit both parallelization in PASCO and parallelizable clustering methods.  One could also explore applications of PASCO in other domains, where our coarsening algorithm could be used to accelerate other graph mining algorithms, such as visualization or node classification. For these tasks, similar principles of parallelization and fusion exploited in \addnew{the PASCO clustering pipeline} could be applied. 
}

\section*{Acknowledgements}

This work has been supported by the DATAREDUX project, ANR19-CE46-0008, and the ANR DARLING project ANR-19-CE48-0002. This project was supported in part by the AllegroAssai ANR project ANR-19-CHIA0009.

\bibliography{sn-bibliography}% common bib file

\newpage

\begin{appendices}

% !TEX root = ../main.tex

\section{Finding consensus between partitions \label{sec:consensus}}

In this section we describe the different approaches that we investigated for aligning different partitions into a reference partition. 
We recall that the problem reformulates as finding the partition $\clustmatbar$ that best agrees with all $(\clustermatrix_r)$ with ${r \in \integ{R}}$. 
Methods from the clustering ensemble literature already propose to solve such a problem. Mathematically, they amount to find a Frechet mean of the partitions $(\clustermatrix_r)$ for various notions of divergence $\dist$ between partitions
\begin{equation}
\label{eq:general_bary_pb}
\clustermatrix \in \underset{\clustmatbar \in \mathcal{P}_{N, \kbar} }{\arg \min}  \frac{1}{R} \sum\limits_{i=1}^R \dist(\clustermatrix_r , \clustmatbar) \,.
\end{equation}
Depending on the choice of divergence $\dist$, we can obtain different consensus methods, as detailed below. 

\paragraph{Linear regression and many-to-one} These two methods are based on a similar notion of divergence 
between partitions. Given two partitions $\clustermatrix \in \mathcal{P}_{N, k}, \clustmatbar \in \mathcal{P}_{N, \kbar}$, it is defined as
\begin{equation}
\label{eq:permut_distance} 
\dist( \clustermatrix, \clustmatbar) = \min_{Q \in \Qcal(k, \kbar)} \| \clustermatrix Q - \clustmatbar \|_{\fro}^2
\end{equation}
Depending on the choice for the set $\Qcal(k, \kbar)$ we get different methods:
\begin{itemize}
    \item When we simply set $\Qcal(k, \kbar) = \Qcal_{\operatorname{lin-reg}}(k, \kbar) = \R^{k \times \kbar}$, 
%solving this problem amounts to solve a linear regression thus we get that 
the optimal matrix $Q$ is given by (see \Cref{lemma:least_square})
%\RG{I would expect $(\clustermatrix^\top \clustermatrix)^{-1}\clustermatrix^\top\clustmatbar$?}
\begin{equation}
\label{eq:sol_lin_reg}
Q_{\operatorname{lin-reg}} =(\clustermatrix^\top \clustermatrix)^{-1}\clustermatrix^\top\clustmatbar= \diag(\clustermatrix^\top \one_N)^{-1} \clustermatrix^\top \clustmatbar \,,
\end{equation}
where $\diag(\clustermatrix^\top \one_N)^{-1}$ 
corresponds to the diagonal matrix containing the inverse of the cluster sizes. This realignment is proposed in \cite{ayad2010voting}.
Even though it allows to compute the divergence $\dist( \clustermatrix, \clustmatbar)$, this matrix yields a ‘‘re-aligned partition'' $\clustermatrix Q_{\operatorname{lin-reg}}$ which is only a ‘‘soft-partition'', with elements in $[0,1]$. We refer to this solution as \texttt{lin-reg}.
\item If one wants to obtain a true partition (with elements in $\{0,1\}$), a solution is to restrict the matrix in $\Qcal(k, \kbar)$ to send each cluster of $\clustermatrix$ to at most one cluster of $\clustmatbar$. For that we define $\Qcal_{\operatorname{many-to-one}}(k, \kbar) = \{Q \in \{0,1\}^{k \times \kbar}, Q \one_{\kbar} = \one_{k} \}$. The solution of \eqref{eq:permut_distance} with $\Qcal_{\operatorname{many-to-one}}(k, \kbar)$ is given (see \Cref{lemma:many_to_one_sol}) by $\operatorname{row-bin}(Q_{\operatorname{lin-reg}})$, where $\operatorname{row-bin}$ is the operator that returns a binary matrix of same shape, where each row contains only zeros except at the position of the maximum in the corresponding row of the input matrix.  We refer to this solution as \texttt{many-to-one}.
\end{itemize}
\paragraph{OT-based method}
To prevent empty clusters in the realigned partition, we also consider $\Qcal(k,\kbar) = \Qcal_{\operatorname{ot}}(k,\kbar)$ in \eqref{eq:permut_distance}. With these constraints, the alignment problem becomes an OT problem which can be related to a specific \emph{quadratic} regularized OT problem \cite{blondel2018smooth} which admits efficient convex solvers, as detailed in the following lemma (proof can be found in \Cref{sec:proof}).
\begin{restatable}{lemma}{alignmentasquadot}
\label{lemma:alignmentasquadot}
Let ${\clustermatrix} \in \mathcal{P}_{N, k}$ and $D = \diag(\clustermatrix^\top \one_N)$. Then problem \eqref{eq:permut_distance} with $ \Qcal(k, \kbar) = \Qcal_{\operatorname{ot}}(k, \kbar)$ is equivalent to the problem
\begin{equation}
\label{eq:quad_ot_problem}
\min_{Q \in \Qcal_{\operatorname{ot}}(k, \kbar)} \|D^{\frac{1}{2}} Q\|_F^2 - 2 \langle C, Q \rangle\,,
\end{equation}
where $C = \clustermatrix^\top \clustmatbar$. Assuming no empty cluster in the partition $\clustermatrix$, the solution of \Cref{eq:quad_ot_problem} is unique and can be solved by considering the dual problem
\begin{equation}
\label{eq:dual_problem}
\max_{\mu \in \R^{k}, \ \nu \in \R^{\kbar}} \ \frac{1}{k} \sum_i \mu_i + \frac{1}{\kbar} \sum_j \nu_j -\frac{1}{4}\|D^{-\frac{1}{2}}[\mu \oplus \nu +2 C]_{+}\|_F^2\,,
\end{equation}
where $\mu \oplus \nu:= (\mu_i + \nu_j)_{ij}$ and for any matrix $A, \ [A]_{+}  := \left(\max\{A_{ij}, 0\}\right)_{ij}$. More precisely, the optimal solution $Q^\star$ of \Cref{eq:quad_ot_problem} can be written as $Q^{\star} = \frac{1}{2} D^{-1}[\mu^\star \oplus \nu^\star + 2 C]_{+}$ where $(\mu^\star, \nu^\star)$ is the optimal solution of \Cref{eq:dual_problem}.
\end{restatable}
Building upon this result we can solve \Cref{eq:quad_ot_problem} by tackling the dual \Cref{eq:dual_problem} which is a convex unconstrained problem of two variables (maximization of a concave function). This expression allow us to use any convex solver, and, as suggested in \cite{blondel2018smooth}, we rely on L-BFGS \cite{liu1989limited} that we find particularly effective in practice. We call this alignment procedure \texttt{quad-ot} which has roughly a $\Ocal(N K^2)$ complexity. 
\begin{remark}
The only difference between \Cref{eq:quad_ot_problem} and standard quadratic OT problem of the form $\min_Q \langle M, Q\rangle + \frac{\gamma}{2} \|Q\|_F^2$ is that in our case the regularization term is $\|D^{1/2}Q\|_F^2$. This is equivalent to consider a Mahalanobis type regularization instead of a $\ell^2_2$ one.
\end{remark}
\paragraph{Solving for the barycenter}
To solve the barycenter problem in \eqref{eq:general_bary_pb} with a distance of the form of \eqref{eq:permut_distance}, we alternate between finding the alignment matrices $Q_r$ as explained above and updating the reference $\clustmatbar$ as in \eqref{eq:fusion_update_ref}. 
This corresponds to an alternating minimization algorithm, where we alternate between (i) realigning the $R$ partitions on the reference and (ii) updating the reference. The reference update is based on \Cref{lemma:fusion_update} which states that finding the closest partition matrix to a set a (realigned) matrices is given by the majority-vote update, see Eq.~\eqref{eq:fusion_update_ref}.

\section{Other edge sampling rules}
\label{app:edge_sampling}

In this section, we present edge sampling rules that we came up with while trying to design an fast coarsening procedure. They were not satisfying but we choose to present them here and explain their drawbacks, as we believe it is informative to better understand the coarsening algorithm we propose in the end. To the best of our knowledge, these strategies were not considered systematically in previous works to coarsen graphs.
\paragraph{Uniform edge sampling} This approach might be the most natural one. It consists in choosing edges uniformly at random in the graph and collapsing them. Doing so favors edges incident to high degree nodes, resulting after collapsing to an even higher degree hypernode. This amplification phenomenon yields an unbalanced final coarsened graph that contains one huge hypernode and all other hypernodes containing only a few nodes. This is problematic as it would not express well the community structure of the initial graph. 

\paragraph{Uniform edge sampling with marked neighboring edges} This approach fixes the above issue. To avoid collapsing onto almost always the same hypernode, when an edge $(i,j)$ is collapsed, we mark the edges incident to nodes $i$ and $j$, and we sample edges uniformly at random among unmarked edges. This solves the issue of unbalancedness. However, the algorithm quickly runs out of unmarked edges and forces the early computation of the next current coarsened graph. 
While this can be done with sparse matrix products between the adjacency matrix and the matrix encoding the composition of the hypernodes, it remains costly and should be avoided as much as possible. 

\paragraph{Uniform edge sampling with marked visited nodes} Here, we want to relax the limit imposed by the previous approach with marked edges. First, recall that sampling an edge uniformly at random is equivalent to sampling a node $i$ with a probability proportional to its degree and sampling a neighbor $j$ uniformly, see \eg, \cite[Section 6.14]{newman2010networks}. So instead of discarding edge $(i,j)$ whenever either $i$ or $j$ has been used in a previous collapse, a natural relaxation is to reject the edge $(i,j)$ only when $i$ has been previously involved in a collapse, {\em irrespective of whether $j$ was involved or not in such a collapse}. Simulations showed that when using this sampling strategy (in step 6 of \Cref{alg:coarsening_one}) to generate coarsening tables $h^\ell$, the hypernodes size distribution was similar to the one with the edge marking strategy, but resulted in much less intermediate coarsened graph reconstructions. 

\paragraph{Uniform  node sampling with marked visited nodes} A final improvement to speedup the sampling procedure is to sample the first node $i$ {\em uniformly at random} among non-visited nodes instead of according to its degree (the second node $j$ being still draw uniformly among the neighbors of $i$). Computationally, this avoids updating the degrees after each collapse and further speeds up the coarsening procedure. The impact of sampling uniformly instead of according to the degrees stays limited thanks to the friendship paradox. This results in \Cref{alg:coarsening_one}.

\addblocknew{
\section{Complexity of the coarsening phase in PASCO \label{app:complexity}}
\begin{lemma}
    The complexity of \Cref{alg:coarsening_one} is $\Ocal(n^{(\ell)} + N_{\operatorname{while}}^{(\ell)} + E^{\ell} )$ where $N_{\operatorname{while}}^{(\ell)}$ is the number of time the while loop is being executed and $E^{(\ell)}$ is the number of non-zero coefficients in the adjacency matrix $A^{(\ell)}$.
    \label{lem:cplx_algo2}
\end{lemma}

\begin{proof}
    The initializations of $V_a$ and $h^{(\ell)}$ cost $\Ocal(n^{(\ell)})$. Drawing the nodes $u$ and $v$ costs $\Ocal(1)$ as they are sampled uniformly at random. Updating $h^{(\ell)}$ (step 7) has complexity $\Ocal(1)$. Similarly, updating $V_a$ can be performed in $\Ocal(1)$ by updating a boolean array that keeps tracks of which nodes are in $V_a$. Moreover, relabeling $h^{(\ell)}$ has complexity $\Ocal(n^{(\ell)})$. Finally, computing $A^{(\ell+1)}$ can be performed in $E^{(\ell)}$ by iterating over the non-zero coefficients of $A^{(\ell)}$. Summarizing, we have that the complexity of \Cref{alg:coarsening_one} is $\Ocal(n^{(\ell)} + N_{\operatorname{while}}^{(\ell)} +  E^{\ell})$. 
\end{proof}

Now we prove \Cref{prop:cplx_algo1}.
\begin{proof}[Proof of \Cref{prop:cplx_algo1}]
    The initialization of \Cref{alg:coarsening_global} (essentially creating $h$) has complexity $\Ocal(N)$. 
    %Inside the while 
    Now let us denote by $N_{rep}$ the number of times the \texttt{while} loop is repeated. According to Lemma~\ref{lem:cplx_algo2}, the overall complexity of the \texttt{while} loop in \Cref{alg:coarsening_global} is $\Ocal( N_{rep} ( N + |E| ) )$. This is obtained by using upperbounds of $n^{(\ell)}$ and $N_{\operatorname{while}}^{(\ell)}$ by $N$, and $E^{\ell}$ by $|E|$ where $|E|$ is the number of non-zero coefficient in the initial adjacency matrix $\bigA$.  Without loss of generality, we can always assume that the graph has no isolated node (the clustering problem being irrelevant for these nodes), therefore the complexity of the while loop reduces to $\Ocal( N_{rep} |E|  )$.

    Now let us prove that $N_{rep} \leq 1 + \log \rho$. 
    In \Cref{alg:coarsening_one}, remark that the \texttt{while} loop may stop for two reasons. First, the target size is reached, meaning that $N_{\operatorname{while}}^{(\ell)} = n^{(\ell)}-n$. This only happens for $\ell = N_{rep} -1$. Second, the set of available nodes $V_a$ is empty. In this case, knowing that at each iteration we can remove either 1 or 2 elements from $V_a$ (except at the first iteration where 2 are removed), we have $n^{(\ell)} / 2 \leq  N_{\operatorname{while}}^{(\ell)} \leq n^{(\ell)} -1$. This is the case for every iteration $\ell \leq N_{rep} -2$.
    Thus, for all $\ell \leq N_{rep} -2$, we have that $n^{(\ell+1)} = n^{(\ell)} - N_{\operatorname{while}}^{(\ell)} \leq  n^{(\ell)} / 2$. Therefore, for all $\ell \leq N_{rep}-1$, we have $n^{(\ell)} \leq (1/2)^{\ell} N$ (as $n^{(0)}=N$).
    Now, remark that $n^{(N_{rep}-1)} > n$, otherwise the coarsening would already have stopped. Recalling that $n = \lfloor \rho^{-1} N \rfloor$ we have $\rho^{-1} N \leq (1/2)^{N_{rep}-1} N$, yielding $N_{rep} \leq 1 + \log \rho$.   
\end{proof}
}

\section{Relegated theoretical results \label{sec:proof}}

\begin{lemma}
\label{lemma:least_square}
Let ${\clustermatrix} \in \mathcal{P}_{N, k}$ be a partition matrix (\Cref{def:partition_matrix}), and assume that $\forall i \in \integ{k}, [{\clustermatrix}^\top \one_N]_i \neq 0$. Then, for any integer $\kbar$ and any matrix $\clustmatbar \in \R^{N \times \kbar}$, with $\Qcal(k, \kbar) = \R^{k \times \kbar}$,  $Q^\star = \diag({P}^\top \one_N)^{-1}{\clustermatrix}^\top \overline{\clustermatrix}$ is an optimal solution to problem \eqref{eq:permut_distance}.
\end{lemma}
\begin{proof}
Denoting $f(Q) = \|{\clustermatrix} Q-\overline{\clustermatrix}\|_F^2$. Since ${\clustermatrix}$ is a partiton matrix its columns have pairwise disjoint support and we have $\clustermatrix^\top\clustermatrix=\diag({\clustermatrix^\top}\one_N)$ hence
\begin{equation}
f(Q) = \|\overline{\clustermatrix}\|_F^2 -2 \langle Q, {\clustermatrix}^\top \overline{\clustermatrix}\rangle + \tr(Q^\top {\clustermatrix}^\top {\clustermatrix} Q) = \|\overline{\clustermatrix}\|_F^2 + \|\diag({\clustermatrix}^\top \one_N)^{\frac{1}{2}}Q\|_F^2- 2 \langle Q, {\clustermatrix}^\top \overline{\clustermatrix}\rangle\,.
\end{equation}
The optimization problem is convex, setting the gradient of $f$ to zero gives the solution.
\end{proof}

\begin{lemma}
\label{lemma:many_to_one_sol}
Let ${\clustermatrix} \in \mathcal{P}_{N, k}$ be a partition matrix (\Cref{def:partition_matrix}), $\overline{\clustermatrix} \in \R^{N \times \overline{k}}$,  and $\Qcal(k, \kbar) = \{Q \in \{0,1\}^{k \times \overline{k}}: Q\one_{\overline{k}} = \one_k\}$. Then $Q^{\star}$ defined by
\begin{equation}
\forall i \in \integ{k}, Q^{\star}_{ij} = \begin{cases} 1 & j \in \operatorname{argmax}_{p \in \integ{\overline{k}}} [{\clustermatrix}^\top \overline{\clustermatrix}]_{i p} \\ 0 & \text{otherwise}  \end{cases}\,,
\end{equation}
is an optimal solution to problem \eqref{eq:permut_distance}.
\end{lemma}
\begin{proof}
Denoting $f(Q) = \|\clustermatrix Q-\overline{\clustermatrix}\|_F^2$ and using that ${\clustermatrix}$ is a partition matrix, that $Q_{ij} \in \{0,1\}$ (hence $Q_{ij}^2=Q_{ij}$) and $Q \one_{\overline{k}} = \one_k$ we can rewrite $f$ as
\begin{equation}
\begin{split}
f(Q) &= \|\overline{\clustermatrix}\|_F^2 + \|\diag({\clustermatrix}^\top \one_N)^{\frac{1}{2}}Q\|_F^2- 2 \langle Q, {\clustermatrix}^\top \overline{\clustermatrix}\rangle \\
&= \|\overline{\clustermatrix}\|_F^2 + \sum_{ij} [{\clustermatrix}^\top \one_N]_{i} Q_{ij}^2 - 2 \langle Q, {\clustermatrix}^\top \overline{\clustermatrix}\rangle \\
&= \|\overline{\clustermatrix}\|_F^2 + \sum_{i=1}^{k} \sum_{j=1}^{\overline{k}} [{\clustermatrix}^\top \one_N]_{i} Q_{ij} - 2 \langle Q, {\clustermatrix}^\top \overline{\clustermatrix}\rangle \\
&= \|\overline{\clustermatrix}\|_F^2 + \sum_{i=1}^{k}  [{\clustermatrix}^\top \one_N]_{i} (\sum_{j=1}^{\overline{k}} Q_{ij}) - 2 \langle Q, {\clustermatrix}^\top \overline{\clustermatrix}\rangle \\
&= \|\overline{\clustermatrix}\|_F^2 + \sum_{i=1}^{k}  [{\clustermatrix}^\top \one_N]_{i} - 2 \langle Q, {\clustermatrix}^\top\overline{\clustermatrix}\rangle \\
&= \|\overline{\clustermatrix}\|_F^2 + N - 2 \langle Q, {\clustermatrix}^\top\overline{\clustermatrix}\rangle 
%= \|\overline{\clustermatrix}\|_F^2 + \langle Q, \widetilde{\clustermatrix}^\top (\one_N \one_{\overline{k}} -2 \overline{\clustermatrix})\rangle \,.
\end{split}
\end{equation}
Denoting $C = - {\clustermatrix}^\top \overline{\clustermatrix}$, a solution to problem \eqref{eq:permut_distance} can thus be found by solving 
\begin{equation}
\label{eq:semi_relaxeproblem}
\min_{Q \in \{0,1\}^{k \times \overline{k}}: Q \one_{\overline{k}} = \one_k} \ \langle Q, C \rangle\,.
\end{equation}
Now \Cref{eq:semi_relaxeproblem} is an optimization problem that decouples with respect to the rows of $Q$, \ie\ there are $k$ independent problems per row of $Q$. For each row $i \in \integ{k}$, a solution can be found by choosing any column %the 
index %of the column 
$j$ such that $j \in \operatorname{argmin}_{p \in \integ{\overline{k}}} C_{ip}$ 
%(see for example \cite[Equation 8]{flamary2016optimal}). %But 
This condition is equivalent to find $j$ such that $j \in \operatorname{argmax}_{p \in \integ{\overline{k}}} [{\clustermatrix}^\top \overline{\clustermatrix}]_{i p}$.
\end{proof}

\begin{lemma}
\label{lemma:fusion_update}
Let $\clustermatrix^{(1)}, \cdots, \clustermatrix^{(R)}$ where each $\clustermatrix^{(r)} \in \R^{N \times \overline{k}}$. 
Then a solution to  
\begin{equation}
\label{eq:barycenter_problem_in_proof}
\underset{\clustmatbar \in \mathcal{P}_{N, \kbar}}{\operatorname{min}} \ \frac{1}{R} \ \sum_{r=1}^{R} \|\clustmatbar - \clustermatrix^{(r)}\|_F^2\,.
\end{equation}
is given by
\begin{equation}
\forall i \in \integ{N}, \ [\overline{\clustermatrix}]_{ij} = \begin{cases} 1 & j \in \underset{p \in \integ{\overline{k}}}{\operatorname{argmax}} \ [\sum_{r=1}^{R} \clustermatrix^{(r)}]_{ip} \\ 0 & \text{otherwise}  \end{cases}
\end{equation}
Now let $\clustermatrix^{(1)}, \cdots, \clustermatrix^{(R)}$ where each $\clustermatrix^{(r)} \in \R^{N \times k_r}$ and $Q^{(1)}, \cdots, Q^{(R)}$ be coupling matrices such that each $Q^{(r)} \in \Qcal_{\ot}(k_r, \kbar)$.
Then a solution to 
\begin{equation}
\label{eq:barycenter_ot_problem_in_proof_2}
\underset{\clustmatbar \in \mathcal{P}_{N, \kbar}}{\operatorname{min}} \ \frac{1}{R} \ \sum_{r=1}^{R} \sum_{i,j=1}^{k, \kbar} \|P^{(r)}_{:,i} - \overline{P
}_{:,j}\|_2^2 Q^{(r)}_{i,j}\,.
\end{equation}
is given by
\begin{equation}
\forall i \in \integ{N}, \ [\overline{\clustermatrix}]_{ij} = \begin{cases} 1 & j \in \underset{p \in \integ{\overline{k}}}{\operatorname{argmax}} \ [\sum_{r=1}^{R} \clustermatrix^{(r)} Q^{(r)}]_{ip} \\ 0 & \text{otherwise}  \end{cases}
\end{equation}
\end{lemma}
\begin{proof}
For the first point, take $\clustmatbar$ in the constraints. Since it is a partition matrix we have $\|\clustmatbar\|_F^2 = \sum_{ij} \clustmatbar_{ij}^2 = \sum_{ij} \clustmatbar_{ij} = \one_N^\top  \clustmatbar \one_{\kbar} =  N$. 
Thus problem \eqref{eq:barycenter_problem_in_proof} is equivalent to 
\begin{equation}
\underset{\clustmatbar \in \mathcal{P}_{N, \kbar}}{\operatorname{min}} \  \langle \clustmatbar, - \sum_{r=1}^{R} \clustermatrix^{(r)}\rangle\,.
\end{equation}
As detailed in the proof of \Cref{lemma:many_to_one_sol} a solution can be found by choosing the index of the column $j$ such that $j \in \operatorname{argmin}_{p \in \integ{\overline{k}}} [- \sum_{r=1}^{R} \clustermatrix^{(r)}]_{ip}$ which concludes the proof for the first point.
For the second point, we use that $\clustmatbar$ is a partition matrix and $Q^{(r)}$ are coupling matrices so that
\begin{equation}
    \begin{split}
        \sum_{r=1}^{R} \sum_{i,j=1}^{k, \kbar} \|P^{(r)}_{:,i} - \clustmatbar_{:,j}\|_2^2 Q^{(r)}_{i,j} &= \sum_{r=1}^{R} \sum_{i,j=1}^{k, \kbar} (\|P^{(r)}_{:,i}\|_2^2 - 2 \langle P^{(r)}_{:,i}, \clustmatbar_{:,j}\rangle  + \|\clustmatbar_{:,j}\|_2^2)Q^{(r)}_{i,j} \\
        &=\text{cte} - 2 \sum_{r=1}^{R} \sum_{i,j=1}^{k, \kbar} \langle P^{(r)}_{:,i}, \clustmatbar_{:,j}\rangle Q^{(r)}_{i,j} + \sum_{r=1}^{R} \sum_{i,j=1}^{k, \kbar} \|\clustmatbar_{:,j}\|_2^2 Q^{(r)}_{i,j} \\
        &=\text{cte} - 2 \sum_{r=1}^{R} \langle \clustmatbar, P^{(r)} Q^{(r)} \rangle  + \sum_{r=1}^{R} \sum_{j=1}^{\kbar} \|\clustmatbar_{:,j}\|_2^2 \sum_{i=1}^{k} Q^{(r)}_{i,j} \\
        &=\text{cte} - - 2 \sum_{r=1}^{R} \langle \clustmatbar, P^{(r)} Q^{(r)} \rangle + \frac{R}{\kbar}\|\clustmatbar\|_\fro^2\,.
    \end{split}
\end{equation}
Using that $\|\clustmatbar\|_\fro^2 = N$ as previously proved, we get that the problem is equivalent to 
\begin{equation}
\underset{\clustmatbar \in \mathcal{P}_{N, \kbar}}{\operatorname{min}} \  \langle \clustmatbar, - \sum_{r=1}^{R} P^{(r)} Q^{(r)}\rangle\,,
\end{equation}
hence the result.
\end{proof}

\alignmentasquadot*
\begin{proof}
We will prove a slightly more general result by considering the problem
\begin{equation}
\label{eq:general_quad_ot}
\min_{Q \in \Qcal_{\ot}(k_r, \kbar)} \ \langle M, Q \rangle + \frac{\gamma}{2}\|L^{\frac{1}{2}} Q\|_F^2\,,
\end{equation}
where $M \in \R^{p \times p_{\reff}}, \gamma > 0$ and $L$ is a symmetric positive definite matrix. 
We note $a = \frac{1}{k} \one_k, b = \frac{1}{\kbar} \one_{\kbar}$.
We will then apply to $M = -2C, \gamma = 2$ and $L = D$ which is symmetric positive definite when there is no empty clusters (since in this case $\forall i \in \integ{K}, D_{ii} \neq 0$). 
Most of our calculus are adapted from \cite{blondel2018smooth}. 
First, since $L$ is a symmetric positive definite matrix, the problem \Cref{eq:general_quad_ot} is a strongly convex problem, thus it admits a unique solution.

To look at the dual of \Cref{eq:general_quad_ot}, we consider the Lagrangian 
\begin{equation}
\begin{split}
\Lcal(Q, \mu, \nu, \Gamma) &= \langle M, Q \rangle + \frac{\gamma}{2}\|L^{\frac{1}{2}} Q\|_F^2 + \langle \mu, a - Q\one_k \rangle + \langle \nu, b - Q^\top \one_{\overline{k}} \rangle - \langle \Gamma, Q\rangle \\
&= \langle M, Q \rangle + \frac{\gamma}{2}\|L^{\frac{1}{2}} Q\|_F^2 - \langle Q, \mu \one_{\overline{k}}^\top + \one_k \nu^\top + \Gamma \rangle  + \langle \mu, a \rangle + \langle \nu, b\rangle\,,
\end{split}
\end{equation}
where $\Gamma$ is the variable accounting for the non-negativity constraints on $Q$. We have
\begin{equation}
\nabla_Q \Lcal(Q, \mu, \nu, \Gamma) =  0 \iff M + \gamma L Q - \mu \oplus \nu - \Gamma = 0\,.
\end{equation}
This is statisfied when $Q = Q^\star = \frac{1}{\gamma} L^{-1}(\Gamma+ \mu \oplus \nu - M)$. Moreover,
\begin{equation}
\begin{split}
\langle M, Q^\star \rangle - \langle Q^\star, \mu \otimes \nu + \Gamma \rangle &= -\langle Q^\star, \mu \otimes \nu + \Gamma -M\rangle \\
&=-\langle \frac{1}{\gamma} L^{-1}(\Gamma+ \mu \oplus \nu - M), \mu \otimes \nu + \Gamma -M\rangle \\
&= -\frac{1}{\gamma} \|L^{-\frac{1}{2}}(\Gamma+ \mu \oplus \nu - M)\|_F^2\,.
\end{split}
\end{equation}
Thus
\begin{equation}
\begin{split}
\langle M, Q^\star \rangle - \langle Q^\star, \mu \otimes \nu + \Gamma \rangle + \frac{\gamma}{2} \|L^{\frac{1}{2}} Q^\star\|_F^2 &= -\frac{1}{\gamma} \|L^{-\frac{1}{2}}(\Gamma+ \mu \oplus \nu - M)\|_F^2 \\
&+ \frac{\gamma}{2} \|L^{\frac{1}{2}} (\frac{1}{\gamma} L^{-1}(\Gamma+ \mu \oplus \nu - M))\|_F^2\\
&=-\frac{1}{\gamma} \|L^{-\frac{1}{2}}(\Gamma+ \mu \oplus \nu - M)\|_F^2 \\
&+ \frac{1}{2\gamma} \|L^{-\frac{1}{2}}(\Gamma+ \mu \oplus \nu - M)\|_F^2 \\
&=-\frac{1}{2\gamma} \|L^{-\frac{1}{2}}(\Gamma+ \mu \oplus \nu - M)\|_F^2
\end{split}
\end{equation}
Hence
\begin{equation}
\Lcal(Q^\star, \mu, \nu, \Gamma) = -\frac{1}{2\gamma} \|L^{-\frac{1}{2}}(\Gamma+ \mu \oplus \nu - M)\|_F^2 + \langle \mu, a\rangle + \langle \nu, b\rangle\,.
\end{equation}
Now we solve the problem over $\Gamma$ that is we maximize the problem
\begin{equation}
\max_{\Gamma \geq 0} \ \Lcal(Q^\star, \mu, \nu, \Gamma)\,,
\end{equation}
where $\geq 0$ should be understood pointwise. This is equivalent to 
\begin{equation}
\label{eq:nn_sqaures_psd}
\min_{\Gamma \geq 0} \|L^{-\frac{1}{2}}(\Gamma-A)\|_F^2\,.
\end{equation}
where $A := M - \mu \oplus \nu$. Writing $L = U \Delta U^\top$ where $\Delta = \diag(d_1, \cdots, d_n)$ with $d_i > 0$, \Cref{eq:nn_sqaures_psd} equivalently writes
\begin{equation}
\min_{\Gamma \geq 0} \|\Delta^{-\frac{1}{2}}\Gamma-\Delta^{-\frac{1}{2}}A\|_F^2\,.
\end{equation}
With a change of variable $\tilde{\Gamma} = \Delta^{-1/2} \Gamma \geq 0$ this is equivalent to
\begin{equation}
\min_{\tilde{\Gamma} \geq 0} \|\tilde{\Gamma}-\Delta^{-\frac{1}{2}}A\|_F^2\,,
\end{equation}
whose minimum is given by $\tilde{\Gamma} = [\Delta^{-\frac{1}{2}}A]_+ = \Delta^{-\frac{1}{2}}[A]_+$ since $\Delta^{-\frac{1}{2}}$ is a diagonal matrix with positive entries. Thus the solution of \eqref{eq:nn_sqaures_psd} is given by $\Gamma = \Delta^{1/2} \tilde{\Gamma} = \Delta^{1/2} [\Delta^{-1/2}A]_+ = [A]_+$.  Also $[A]_+ - A = [-A]_{+}$. Thus 
\begin{equation}
\min_{\Gamma \geq 0} \|L^{-\frac{1}{2}}(\Gamma-A)\|_F^2 = \|L^{-\frac{1}{2}}[-A]_{+}\|_F^2 = \|L^{-\frac{1}{2}}[\mu \oplus \nu -M]_{+}\|_F^2
\end{equation}
Hence the dual problem of \Cref{eq:general_quad_ot} is given by $\max_{\mu, \nu} \Lcal(Q^\star, \mu, \nu, [-A]_{+})$ which is
\begin{equation}
\max_{\mu, \nu} \ \langle \mu, a \rangle + \langle \nu, b\rangle -\frac{1}{2 \gamma}\|L^{-\frac{1}{2}}(\mu \oplus \nu -M)_{+}\|_F^2\,.
\end{equation}
Applying this to $M = -2C, \gamma = 2$ and $L = D$ concludes.
\end{proof}

\section{Experiments details and extra results}

\subsection{Details about the implementation \label{app:implem_details}}

%The figures are generated using \texttt{Matplotlib} \cite{matplolib}. 
% The PASCO implementation relies on \texttt{Scipy} and \texttt{Numpy} \cite{2020SciPyNMeth,harris2020array} as well as the \texttt{POT} library \cite{flamary2021pot}.
The PASCO implementation relies on the \texttt{POT} library \cite{flamary2021pot} for the fusion part. 
The heaviest experiments were performed with the support of the Centre Blaise Pascal's IT test platform at ENS de Lyon (Lyon, France) that operates the SIDUS solution \cite{quemener2013sidus}. We use an \emph{Intel Xeon Gold 5218} machine.

\subsection{Details of the coarsening experiment.}

\Cref{tab:graph_caracteristics} provide some characteristics of the real graphs used in \Cref{fig:RSA}.

\begin{table}[h]
\centering
\begin{tabular}{l|c|c|c|c|c}
          & \# nodes & \# edges & avg degree & assortativity & avg clustering coef \\ \hline
Yeast     &  1.5k    &   1.9k   &     2      &      -0.21    &         0.07        \\ \hline
Minnesota &  2.6k    &   3.3k   &     2      &      -0.18    &         0.02        \\ \hline
Airfoil   &  4.3k    &  12.3k   &     5      &       0.32    &         0.41 
\end{tabular}
\caption{Some characteristics of the real graphs used in \Cref{fig:RSA}, extracted from \cite{nr}.}
\label{tab:graph_caracteristics}
\end{table}

The RSA experiments in \Cref{subsec:expe_coarsening} tested the conservation of spectral properties by coarsening algorithms, including PASCO. \Cref{fig:RSA_SBM} complete the \Cref{fig:RSA} with SBM graphs. Graphs were drawn under the $\SSBM(1000,k,2,\alpha)$, for  $(k = 10, \alpha = 1/k)$, $(k = 100, \alpha = 1/k)$ and $(k = 10, \alpha = 1/(2k))$. Below, in \Cref{fig:RSA_timings}, we display computational times of the coarsening methods. 

\begin{figure}[h]
    \centering
    \includegraphics[width=\linewidth]{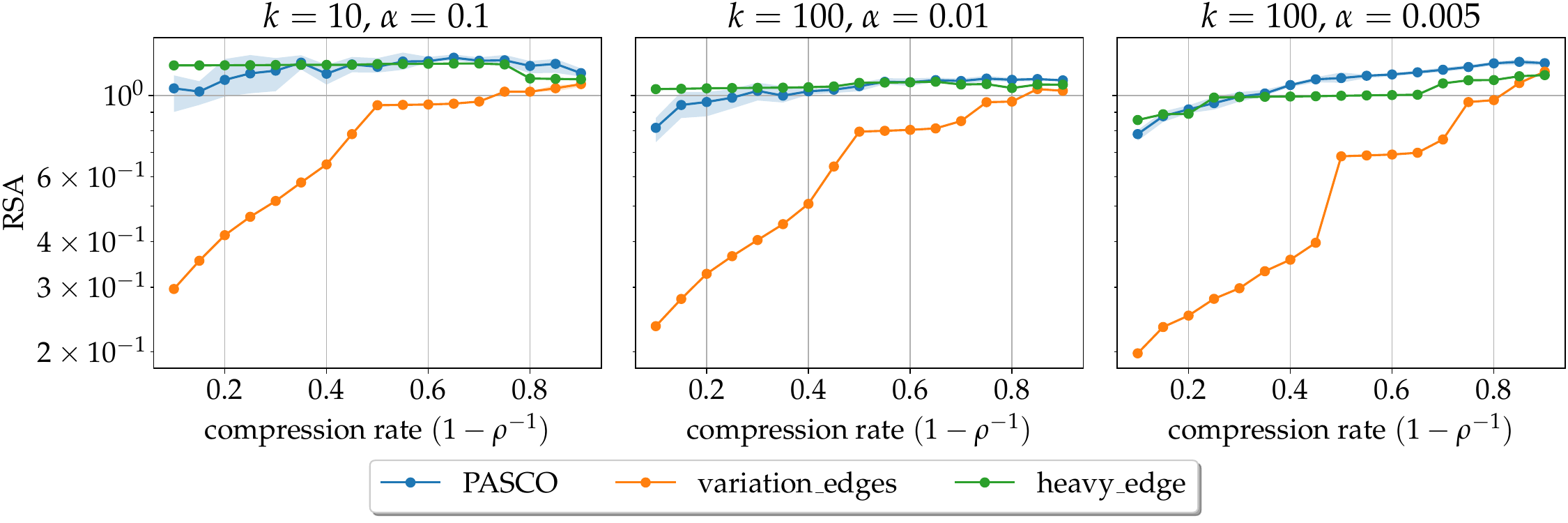} % Replace with your image path
    \caption{We represent the RSA (the smaller, the better) of various coarsening schemes (including PASCO), as a function of the compression rate (the higher, the coarser the obtained graph). Shaded areas represent 0.2 upper- and lower-quantiles. 
    %already said: Top row: we reproduce a part of the experiment of \cite[Figure 2]{loukas2019graph} with the same real graphs and added PASCO. Bottom row: same experiment but with random graphs drawn from $\SSBM(1000, k, 2, \alpha)$ for $(k=10, \alpha=1/k)$, $(k=100, \alpha=1/k)$ and $(k=10, \alpha=1/(2k))$.
    }
    \label{fig:RSA_SBM}
\end{figure}

\begin{figure}[h]
    \centering
    \includegraphics[width=.9\linewidth]{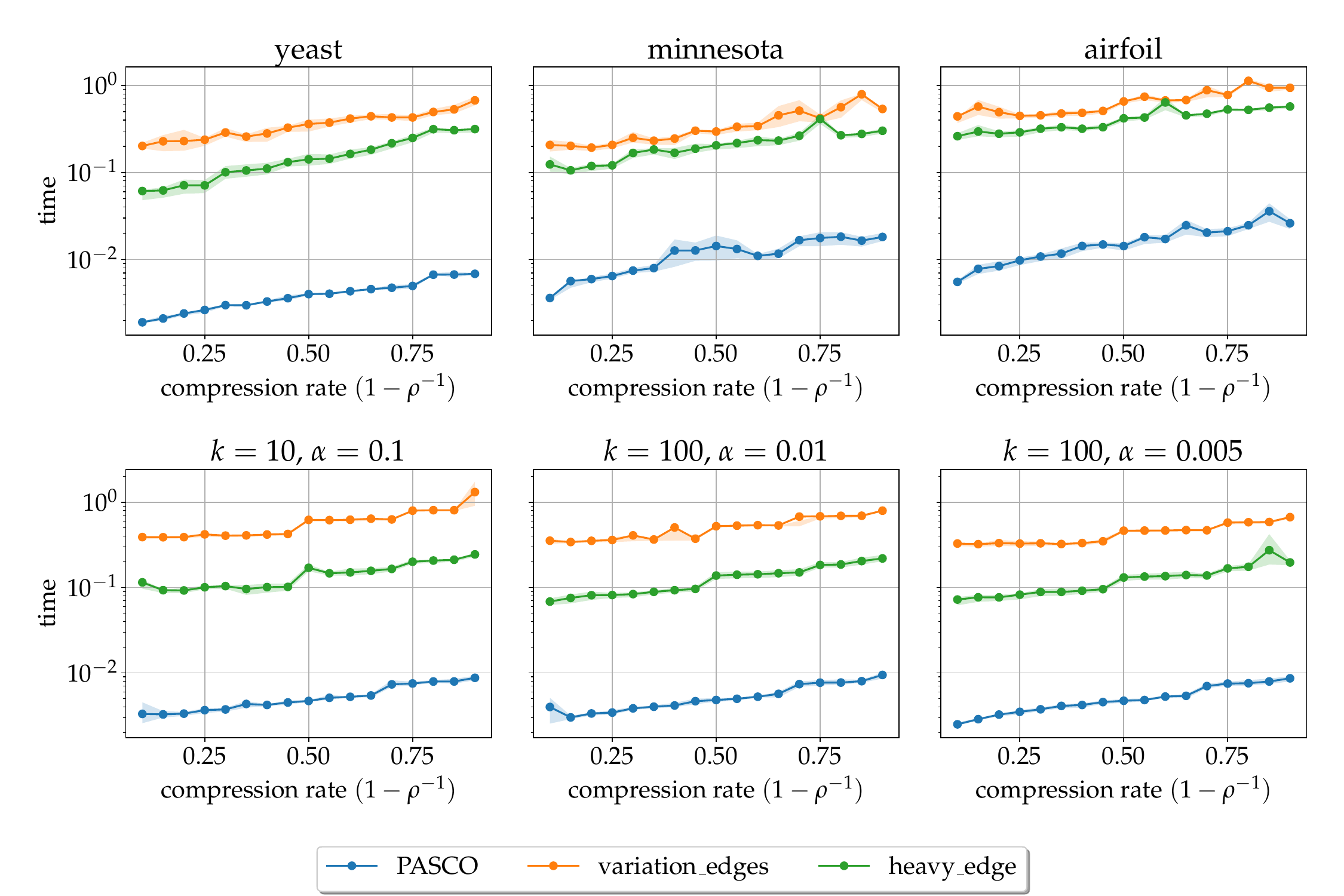} % Replace with your image path
    \caption{We represent the computational time of various coarsening schemes (including PASCO), as a function of the compression rate (the higher, the coarser the obtained graph). Shaded areas represent 0.2 upper- and lower-quantiles over the 10 repetitions of the experiment. Top row: we reproduce a part of the experiment of \cite[Figure 2]{loukas2019graph} with the same real graphs and added PASCO. Bottom row: same experiment but with random graphs drawn from $\SSBM(1000, k, 2, \alpha)$ for $(k=10, \alpha=1/k)$, $(k=100, \alpha=1/k)$ and $(k=10, \alpha=1/(2k))$.}
    \label{fig:RSA_timings}
\end{figure}

\FloatBarrier
\subsection{Effectiveness of the alignment/fusion phase}
\label{subsec:expe_fusion}

To demonstrate the alignment+fusion procedure, we now consider a synthetic dataset generated from a two-dimensional Gaussian Mixture Model (GMM) with three clusters. The clusters consist of 500, 400 and 200 points, respectively. Each cluster is sampled from isotropic Gaussian distributions with a standard deviation of 0.25 and centers located at $(0,1)$, $(1,0)$, and $(0,0)$. The resulting dataset is shown in the top left panel of \Cref{fig:illustration_fusion}.

%Here there is no graph as we are focusing on the alignment and fusion of partitions (that already used graph input), and the methods we consider for step 3 don't use information on graph.
We generate 15 different partitions of the dataset with the goal of recovering the true partition corresponding to the original GMM clusters. The partitions are constructed as follows: first, we randomly select a number of clusters $k$, drawn uniformly between 3 and 10. Then, we designate $k$ centroids: the first three are randomly selected from each of the true GMM clusters (so that we have at least one starting centroid in each true cluster), while the remaining centroids are uniformly sampled from the remaining points. Each point in the dataset is assigned to the nearest centroid, thereby forming a partition. 
Forcing each true cluster to be initially represented by at least one centroid ensures that the resulting partition is related to the true partition. 
%However, note that the drawn centroids can lie at the edge of the clusters and would lead to a partition very different to the true partition. 
See two examples of these generated partitions in \Cref{fig:illustration_fusion}.

The effectiveness of the proposed alignment and fusion method (\Cref{alg:alginment_fusion}) is evaluated by comparing several alignment techniques: \texttt{lin-reg}, \texttt{many-to-one}, and the proposed \texttt{ot} to recover the true partition. For each case, we consider a randomly initialized reference with $\overline{k} = 3$, such that each point is assigned to a cluster chosen uniformly at random. As a baseline, we compare with a $K$-means clustering with $K=3$. We emphasize that the $K$-means algorithm benefits from the spatial coordinates of the data points, whereas the alignment+fusion methods operate solely on the different partitions $P_1, \cdots, P_R$, and ignore the positions. 
The experiment is repeated \add{hundred \sout{five}} times, and the average Adjusted Mutual Information (AMI) \cite{vinh2009information} (see \Cref{app:scores} for the definition) between the inferred and true partitions is plotted as a function of the number of partitions $R$ on the bottom right panel of \Cref{fig:illustration_fusion}. The results indicate that the OT methods achieve performance comparable to $K$-means (high AMI, with small variance) when $R \geq 6$, while \texttt{lin-reg} and \texttt{many-to-one} have high variance and struggle to retrieve the true partition for any value of $R$. This can be explained by the fact that the partitions are quite unbalanced and thus more suited for the coupling constraints. Finally, a partition recovered using the \texttt{ot} method is shown in the bottom left panel of \Cref{fig:illustration_fusion}, illustrating that it is nearly a perfect permutation of the true partition. 

\begin{figure}[t]
\centering
\includegraphics[width=\linewidth]{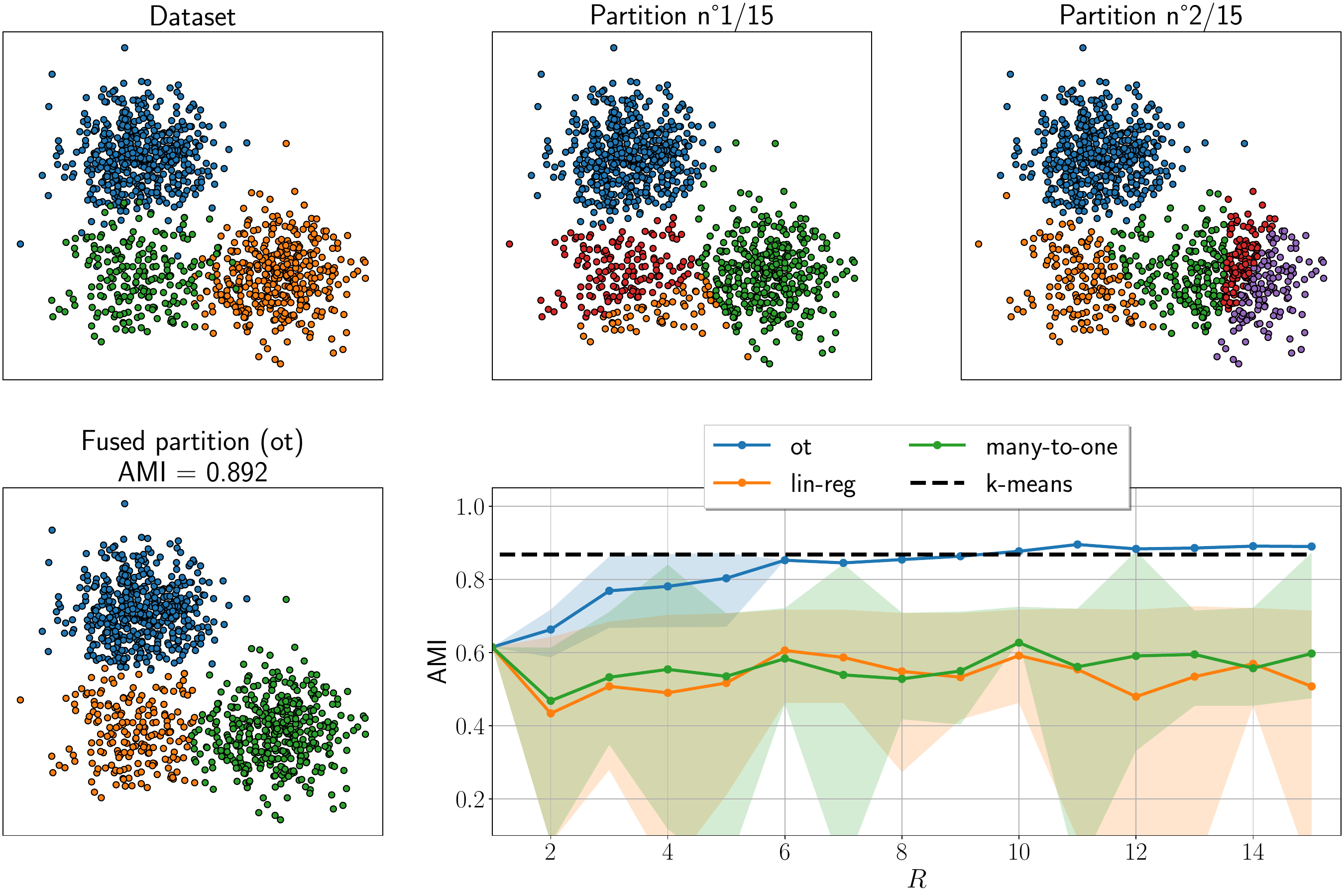}
\caption{\label{fig:illustration_fusion} Experience with a toy dataset drawn from a $2D$ GMM (top left). Colors indicate clusters. Two partitions (out of 15) are depicted (top center and right), as well as the recovered partition using the alignment+fusion procedure based on \texttt{ot} (bottom left). The bottom right panel presents the average AMI  between the true partition and the fused partitions obtained by \texttt{ot}, \texttt{many-to-one}, and \texttt{lin-reg}. Shaded areas represent 0.2 upper- and lower-quantiles over the 100 runs of the alignment+fusion algorithms.} 
\end{figure} 

\FloatBarrier

\subsection{Additional results on parameters influence}

Here, we provide additional results on the experiments of \Cref{subsec:expe_SBM}. \Cref{subfig:influence_align} show the performance of the various alignment methods for different difficulty levels in the SSBM. Parameters are the same as in \Cref{fig:param_influence}.

In \Cref{fig:pasco_timing}, we showed the computational gains of using PASCO. As a sanity check, in \Cref{subfig:quality_for_timing} we show that the performance \emph{w.r.t.} the quality of the output partition are satisfying. Speed was not achieve at the cost of quality.

\begin{figure}
    \centering
    \includegraphics[width=0.6\linewidth]{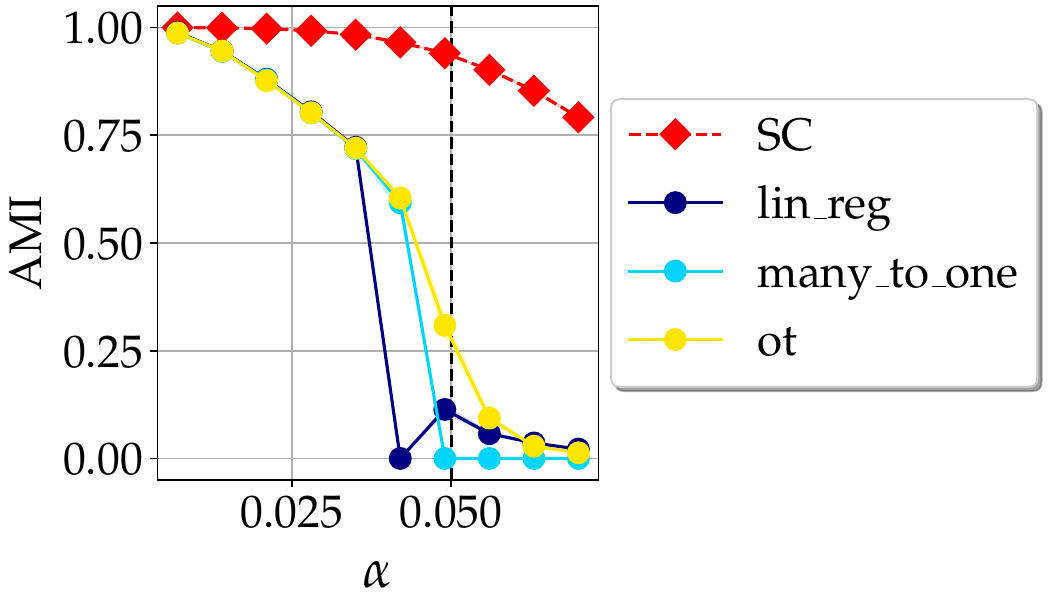}
    \caption{Influence of the alignment method in PASCO. The AMI score is average over the 10 runs and shaded areas represent 0.2 upper- and lower-quantiles. Dashed lines correspond to PASCO threshold}
    \label{subfig:influence_align}
\end{figure}
\begin{figure}
    \centering
    \includegraphics[width=0.6\linewidth]{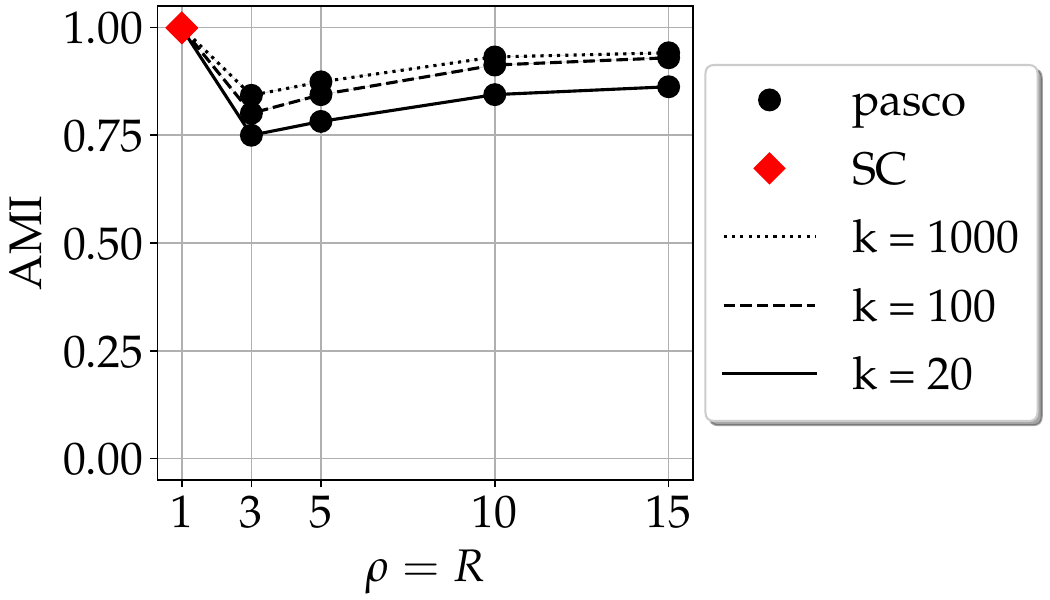}
    \caption{Quality of the partitions w.r.t. $\rho$ and $R$, when $R=\rho$. The number of communities vary in $\{20, 100, 1000\}$. AMI values are average over 10 runs and shaded areas represent 0.2 upper- and lower-quantiles.}
    \label{subfig:quality_for_timing}
\end{figure}

\FloatBarrier

\subsection{Definition of the scores used to evaluate clustering quality \label{app:scores}}

\begin{definition}[Adjusted Mutual Information]
The Adjusted Mutual Information (AMI) between two partitions $U = (U_1, \dots, U_k)$ and $V = (V_1, \dots, V_{k'})$ of the set of node $\bigV$ of size $N$ is given by
\begin{equation}
AMI(U,V) = \frac{MI(U,V) - \mathbb{E}[MI(U,V)]}{  \max(H(U), H(V)) - \mathbb{E}[MI(U,V)]} \,,
\label{eq:AMI}
\end{equation}
where 
\begin{equation}
MI(U,V) = \sum\limits_{i=1}^k \sum\limits_{j=1}^{k'} P_{UV}(i,j) \log \frac{P_{UV}(i,j)}{P_U(i) P_V(j)} \,,
\label{eq:MI}
\end{equation}
with $P_U(i) = |U_i|/N$ (similarly for $P_V(j)$), $P_{UV}(i,j) = |U_i \cap V_j|/N$, and $H(U) = - \sum_i P_U(i) \log P_U(i)$ (similarly for $H(V)$). The expected mutual information (MI) $\mathbb{E}[MI(U,V)]$ in \eqref{eq:AMI} is computed by assuming hyper-geometric distribution for $U$ and $V$, the parameters being estimated from the partitions. For the sake of conciseness and simplicity, we refer the reader to \cite{vinh2009information} for the precise formula.
\label{def:AMI}
\end{definition}

\begin{definition}[Generalized Normalized Cut]
Given a graph $\bigG$ represented by its adjacency or weight matrix $\bigA$, consider a partition $(V_1, \dots, V_k)$ of the vertex set $\bigV$ of $\bigG$. The generalized normalized cut of the partition is defined as 
\begin{equation}
\frac{1}{k} \sum\limits_{j=1}^k \frac{ \sum_{u \in V_j, v \notin V_j} \bigA_{u,v} }{ \sum_{u \in V_j, v \in \bigV} \bigA_{u,v}} \,.
\label{eq:gnCut}
\end{equation}
\label{def:gnCut}
\end{definition}

\begin{definition}[Modularity]
Given a graph $\bigG$ represented by its adjacency or weight matrix $\bigA$, consider a partition $(V_1, \dots, V_k)$ of the vertex set $\bigV$ of $\bigG$. Let $d_u$ denote the degree of node $u$ and $m$ the weight of all the edges. Then, the modularity is given by
\begin{equation}
\frac{1}{2m} \sum\limits_{j=1}^k  \sum_{u \in V_j, v \in V_j} \left( \bigA_{u,v} - \frac{d_u d_v}{2m} \right) \,.
\label{eq:modularity}
\end{equation}
\label{def:modularity}
\end{definition}

\begin{definition}[Description Length]
Let $\bigG = (\bigV, \bigE)$ be a graph and consider the partition $(V_1, \dots, V_k)$ of the vertex set $\bigV$. We denote by $e_{i,j}$ the number of edges between $V_i$ and $V_j$. In mathematical terms, the description length is defined by $\texttt{dl} = \mathcal{S} + \mathcal{L}$, where $\mathcal{S}$ is the entropy of the fitted stochastic block model and $\mathcal{L}$ is the information required to describe the model. Their expressions are given by 
\begin{align*}
\mathcal{S} &= \abs{\bigE} - \frac{1}{2} \sum\limits_{i,j=1}^{k} e_{i,j} \log\left( \frac{e_{i,j}}{\abs{V_i} \abs{V_j}} \right) \\
\mathcal{L} &= \abs{\bigE} \cdot h \left( \frac{k (k+1)}{2\abs{\bigE}} \right) + \abs{\bigV} \log(k), 
\end{align*}
with $h(x) = (1+x) \log(1+x) - x \log x$. 
\label{def:dl}
\end{definition}

\subsection{Real data experiments \label{app:real_data_tables}}

\Cref{tab:real_datasets} provides a few caracteristics of the real graphs used in \Cref{subsec:expe_real}.

\begin{table}[h]
    \centering
    \begin{tabular}{l|c|c|c|c|c}
        Name     & \# nodes    &   \# edges    &  $k$  & $\alpha_{est}$  & $1/(k-1)$ \\ \hline
        \texttt{arxiv}    & $169,343$   &   $2,315,598$ &  $40$ & $0.044$ & $0.026$    \\ 
        \texttt{mag}      & $726,664$   &  $10,778,888$ & $349$ & $0.031$ & $0.0029$   \\ 
        \texttt{products} & $2,385,902$ & $123,612,734$ & $47$  & $0.028$ & $0.022$
    \end{tabular}  
    \caption{Name, number of nodes, number of edges and number of communities of the large real datasets used in the experiments.}
    \label{tab:real_datasets}
\end{table}

Here, we present the results of the experiments on real graph with tables. Bold figures represent the best result for each criterion for a given clustering algorithm (\texttt{SC}, \texttt{CSC}, \texttt{louvain}, \texttt{leiden}, \texttt{MDL}, \texttt{infomap}). 

\begin{table}
\centering
\caption{Results for the arxiv dataset.}
\label{tab:arxiv_table}
\begin{tabular}{lccccc}
\toprule
methods & time $\downarrow$ & ami $\uparrow$ & modularity $\uparrow$ & gnCut $\downarrow$ & dl $\downarrow$ \\
\midrule
ground truth &  &  & 0.493 & 0.436 & 9.23e6 \\
\midrule
SC & 7.42e1 & 0.19 & 0.26 & 0.819 & 9.58e6 \\
SC+PASCO $(t=1)$ & \textbf{4.46e0} & 0.294 & 0.353 & 0.81 & 9.46e6 \\
SC+PASCO $(t=3)$ & 1.02e1 & 0.289 & 0.321 & \textbf{0.779} & 9.49e6 \\
SC+PASCO $(t=5)$ & 1.16e1 & 0.31 & 0.353 & 0.839 & \textbf{9.43e6} \\
SC+PASCO $(t=10)$ & 1.87e1 & \textbf{0.311} & \textbf{0.354} & 0.883 & \textbf{9.43e6} \\
SC+PASCO $(t=15)$ & 3.06e1 & 0.31 & 0.34 & 0.893 & 9.44e6 \\
\midrule
CSC & 2.71e2 & 0.129 & \textbf{0.361} & \textbf{ 0.4} & 9.60e6 \\
CSC+PASCO $(t=1)$ & \textbf{2.54e1} & 0.185 & 0.325 & 0.481 & 9.56e6 \\
CSC+PASCO $(t=3)$ & 3.06e1 & 0.192 & 0.299 & 0.477 & 9.58e6 \\
CSC+PASCO $(t=5)$ & 3.28e1 & 0.21 & 0.291 & 0.48 & 9.55e6 \\
CSC+PASCO $(t=10)$ & 5.29e1 & 0.255 & 0.303 & 0.546 & 9.45e6 \\
CSC+PASCO $(t=15)$ & 7.46e1 & \textbf{0.268} & 0.311 & 0.572 & \textbf{9.41e6} \\
\midrule
louvain & 1.77e0 & 0.39 & 0.704 & 0.897 & 8.39e6 \\
louvain+PASCO $(t=1)$ & \textbf{1.14e0} & 0.349 & \textbf{0.581} & 0.881 & \textbf{9.05e6} \\
louvain+PASCO $(t=3)$ & 5.00e0 & 0.37 & 0.615 & \textbf{0.879} & 8.94e6 \\
louvain+PASCO $(t=5)$ & 6.76e0 & 0.391 & 0.637 & 0.913 & 8.88e6 \\
louvain+PASCO $(t=10)$ & 1.22e1 & 0.393 & 0.648 & 0.914 & 8.85e6 \\
louvain+PASCO $(t=15)$ & 1.79e1 & \textbf{0.405} & 0.655 & 0.918 & 8.82e6 \\
\midrule
leiden & 1.38e1 & 0.409 & 0.713 & 0.909 & 8.38e6 \\
leiden+PASCO $(t=1)$ & \textbf{4.77e0} & 0.359 & \textbf{0.579} & 0.891 & \textbf{9.06e6} \\
leiden+PASCO $(t=3)$ & 1.28e1 & 0.391 & 0.619 & 0.885 & 8.95e6 \\
leiden+PASCO $(t=5)$ & 1.66e1 &  0.4 & 0.64 & \textbf{0.883} & 8.85e6 \\
leiden+PASCO $(t=10)$ & 2.86e1 & 0.408 & 0.654 & 0.892 & 8.81e6 \\
leiden+PASCO $(t=15)$ & 3.67e1 & \textbf{0.418} & 0.665 & 0.922 & 8.75e6 \\
\midrule
MDL & 7.51e2 & 0.351 & 0.651 & 0.705 & 8.03e6 \\
MDL+PASCO $(t=1)$ & \textbf{3.82e2} & 0.322 & 0.396 & \textbf{0.438} & 8.93e6 \\
MDL+PASCO $(t=3)$ & 5.24e2 & 0.338 & 0.384 & 0.442 & \textbf{8.98e6} \\
MDL+PASCO $(t=5)$ & 4.78e2 & 0.351 & 0.429 & 0.45 & 8.79e6 \\
MDL+PASCO $(t=10)$ & 6.52e2 & 0.379 & 0.462 & 0.476 & 8.61e6 \\
MDL+PASCO $(t=15)$ & 8.09e2 & \textbf{0.387} & \textbf{0.496} & 0.506 & 8.52e6 \\
\midrule
infomap & 3.14e1 & 0.376 & 0.696 & 0.775 & 8.06e6 \\
infomap+PASCO $(t=1)$ & \textbf{6.04e0} & 0.358 & \textbf{0.568} & \textbf{0.751} & \textbf{8.82e6} \\
infomap+PASCO $(t=3)$ & 1.09e1 & 0.382 & 0.582 & 0.766 & 8.80e6 \\
infomap+PASCO $(t=5)$ & 1.22e1 & 0.395 & 0.605 & 0.782 & 8.67e6 \\
infomap+PASCO $(t=10)$ & 2.32e1 & 0.415 & 0.628 & 0.799 & 8.57e6 \\
infomap+PASCO $(t=15)$ & 2.81e1 & \textbf{0.423} & 0.639 & 0.809 & 8.53e6 \\
\bottomrule
\end{tabular}
\end{table}

\begin{table}
\centering
\caption{Results for the mag dataset.}
\label{tab:mag_table}
\begin{tabular}{lccccc}
\toprule
methods & time $\downarrow$ & ami $\uparrow$ & modularity $\uparrow$ & gnCut $\downarrow$ & dl $\downarrow$ \\
\midrule
ground truth &  &  & 0.268 & 0.217 & 5.35e7 \\
\midrule
SC & 2.18e3 & 0.325 & 0.727 & 0.64 & \textbf{4.92e7} \\
SC+PASCO $(t=1)$ & \textbf{2.90e2} & 0.367 & \textbf{0.672} & 0.665 & 4.80e7 \\
SC+PASCO $(t=3)$ & 3.60e2 & 0.377 & 0.713 & \textbf{0.608} & 4.70e7 \\
SC+PASCO $(t=5)$ & 4.17e2 & 0.393 & 0.749 & 0.652 & 4.63e7 \\
SC+PASCO $(t=10)$ & 7.32e2 & 0.403 & 0.782 & 0.709 & 4.59e7 \\
SC+PASCO $(t=15)$ & 8.90e2 & \textbf{0.406} & 0.791 & 0.757 & 4.56e7 \\
\midrule
CSC & 4.92e3 & 0.145 & \textbf{0.438} & 0.411 & 5.14e7 \\
CSC+PASCO $(t=1)$ & \textbf{6.97e2} & 0.235 & 0.505 & 0.424 & 5.22e7 \\
CSC+PASCO $(t=3)$ & 7.81e2 & 0.24 & 0.532 & \textbf{0.401} & \textbf{5.24e7} \\
CSC+PASCO $(t=5)$ & 9.03e2 & 0.26 & 0.566 & 0.435 & 5.15e7 \\
CSC+PASCO $(t=10)$ & 1.38e3 &  0.3 & 0.609 & 0.514 & 5.04e7 \\
CSC+PASCO $(t=15)$ & 2.20e3 & \textbf{0.328} & 0.646 & 0.56 & 4.96e7 \\
\midrule
louvain & 1.36e1 & \textbf{0.378} & 0.842 & 0.931 & 4.71e7 \\
louvain+PASCO $(t=1)$ & \textbf{8.47e0} & 0.331 & \textbf{0.748} & 0.904 & \textbf{5.03e7} \\
louvain+PASCO $(t=3)$ & 2.32e1 & 0.352 & 0.773 & 0.891 & 4.98e7 \\
louvain+PASCO $(t=5)$ & 3.31e1 & 0.358 & 0.797 & \textbf{0.824} & 4.92e7 \\
louvain+PASCO $(t=10)$ & 8.07e1 & 0.364 & 0.804 & 0.903 & 4.91e7 \\
louvain+PASCO $(t=15)$ & 6.98e1 & 0.366 & 0.815 & 0.871 & 4.89e7 \\
\midrule
leiden & 9.06e1 & \textbf{0.379} & 0.851 & 0.934 & 4.66e7 \\
leiden+PASCO $(t=1)$ & \textbf{3.39e1} & 0.343 & \textbf{0.755} & 0.918 & \textbf{5.01e7} \\
leiden+PASCO $(t=3)$ & 6.74e1 & 0.36 & 0.78 & 0.918 & 4.95e7 \\
leiden+PASCO $(t=5)$ & 8.18e1 & 0.372 & 0.803 & \textbf{0.882} & 4.90e7 \\
leiden+PASCO $(t=10)$ & 1.36e2 & 0.377 & 0.816 & 0.932 & 4.87e7 \\
leiden+PASCO $(t=15)$ & 1.92e2 & \textbf{0.379} & 0.817 & 0.943 & 4.87e7 \\
\midrule
MDL & 3.91e3 & 0.357 & 0.701 & 0.709 & 3.95e7 \\
MDL+PASCO $(t=1)$ & \textbf{3.32e3} & 0.349 & \textbf{0.446} & 0.45 & \textbf{4.53e7} \\
MDL+PASCO $(t=3)$ & 4.54e3 & 0.355 & 0.474 & \textbf{0.444} & 4.50e7 \\
MDL+PASCO $(t=5)$ & 4.11e3 & 0.369 & 0.521 & 0.475 & 4.37e7 \\
MDL+PASCO $(t=10)$ & 4.95e3 & 0.382 & 0.583 & 0.526 & 4.22e7 \\
MDL+PASCO $(t=15)$ & 7.24e3 & \textbf{0.386} & 0.619 & 0.547 & 4.17e7 \\
\midrule
infomap & 2.24e2 & 0.365 & 0.764 & 0.784 & 4.02e7 \\
infomap+PASCO $(t=1)$ & \textbf{4.12e1} & 0.359 & \textbf{0.667} & \textbf{0.674} & \textbf{4.59e7} \\
infomap+PASCO $(t=3)$ & 1.22e2 & 0.372 & 0.719 & 0.704 & 4.56e7 \\
infomap+PASCO $(t=5)$ & 1.63e2 & 0.384 & 0.752 & 0.715 & 4.45e7 \\
infomap+PASCO $(t=10)$ & 2.83e2 & 0.395 & 0.773 & 0.742 & 4.38e7 \\
infomap+PASCO $(t=15)$ & 4.06e2 & \textbf{0.397} & 0.784 & 0.757 & 4.35e7 \\
\bottomrule
\end{tabular}
\end{table}

\begin{table}
\centering
\caption{Results for the products dataset.}
\label{tab:products_table}
\begin{tabular}{lccccc}
\toprule
methods & time $\downarrow$ & ami $\uparrow$ & modularity $\uparrow$ & gnCut $\downarrow$ & dl $\downarrow$ \\
\midrule
ground truth &  &  & 0.728 & 0.464 & 5.28e8 \\
\midrule
SC & 6.37e2 & 0.202 & \textbf{0.603} & 0.722 & \textbf{5.45e8} \\
SC+PASCO $(t=1)$ & \textbf{2.74e2} & 0.327 & 0.41 & 0.717 & 5.97e8 \\
SC+PASCO $(t=3)$ & 4.37e2 & 0.283 & 0.412 & \textbf{0.663} & 5.98e8 \\
SC+PASCO $(t=5)$ & 5.48e2 & \textbf{0.363} & 0.48 & 0.801 & 5.92e8 \\
SC+PASCO $(t=10)$ & 7.45e2 & 0.332 & 0.444 & 0.885 & 5.95e8 \\
SC+PASCO $(t=15)$ & 9.99e2 & 0.32 & 0.42 & 0.824 & 5.96e8 \\
\midrule
CSC & 2.08e4 & 0.206 & 0.585 & 0.601 & 5.43e8 \\
CSC+PASCO $(t=1)$ & \textbf{3.32e3} & 0.275 & 0.584 & 0.569 & 5.49e8 \\
CSC+PASCO $(t=3)$ & 5.07e3 & 0.262 & 0.552 & \textbf{0.526} & 5.59e8 \\
CSC+PASCO $(t=5)$ & 5.26e3 & 0.302 & 0.561 & 0.555 & 5.57e8 \\
CSC+PASCO $(t=10)$ & 6.45e3 & 0.406 & 0.669 & 0.625 & \textbf{5.30e8} \\
CSC+PASCO $(t=15)$ & 7.54e3 & \textbf{0.436} & \textbf{0.715} & 0.718 & 5.24e8 \\
\midrule
louvain & 9.33e1 & 0.523 & 0.873 & 0.955 & 4.62e8 \\
louvain+PASCO $(t=1)$ & \textbf{8.95e1} & 0.49 & \textbf{0.779} & 0.937 & \textbf{5.17e8} \\
louvain+PASCO $(t=3)$ & 2.75e2 & 0.515 & 0.815 & \textbf{0.896} & 5.06e8 \\
louvain+PASCO $(t=5)$ & 3.76e2 & 0.531 & 0.834 & 0.906 & 5.01e8 \\
louvain+PASCO $(t=10)$ & 6.24e2 & \textbf{0.537} & 0.85 & 0.929 & 4.98e8 \\
louvain+PASCO $(t=15)$ & 8.23e2 & 0.535 & 0.849 & 0.897 & 4.97e8 \\
\midrule
leiden & 7.93e2 & \textbf{0.554} & 0.881 & 0.957 & 4.55e8 \\
leiden+PASCO $(t=1)$ & \textbf{3.61e2} & 0.494 & \textbf{0.786} & 0.941 & \textbf{5.14e8} \\
leiden+PASCO $(t=3)$ & 8.07e2 & 0.526 & 0.826 & \textbf{0.925} & 5.05e8 \\
leiden+PASCO $(t=5)$ & 9.63e2 & 0.538 & 0.845 & 0.944 & 5.00e8 \\
leiden+PASCO $(t=10)$ & 1.65e3 & 0.549 & 0.851 & 0.94 & 4.95e8 \\
leiden+PASCO $(t=15)$ & 2.26e3 & 0.547 & 0.858 & 0.949 & 4.96e8 \\
\midrule
MDL & 5.27e4 & 0.494 & 0.859 & 0.887 & 4.56e8 \\
MDL+PASCO $(t=1)$ & \textbf{4.65e4} & 0.491 & \textbf{0.717} & \textbf{0.745} & \textbf{4.92e8} \\
MDL+PASCO $(t=3)$ & 5.93e4 & 0.516 & 0.743 & 0.761 & 4.86e8 \\
MDL+PASCO $(t=5)$ & 6.55e4 & 0.533 & 0.785 & 0.803 & 4.75e8 \\
MDL+PASCO $(t=10)$ & 8.03e4 & 0.551 & 0.808 & 0.799 & 4.70e8 \\
MDL+PASCO $(t=15)$ & 9.58e4 & \textbf{0.56} & 0.827 & 0.837 & 4.65e8 \\
\midrule
infomap & 2.71e3 & 0.504 & 0.87 & 0.896 & 4.54e8 \\
infomap+PASCO $(t=1)$ & \textbf{4.78e2} & 0.495 & \textbf{0.765} & \textbf{0.797} & \textbf{4.89e8} \\
infomap+PASCO $(t=3)$ & 7.61e2 & 0.523 & 0.794 & 0.832 & 4.80e8 \\
infomap+PASCO $(t=5)$ & 9.93e2 & 0.533 & 0.812 & 0.834 & 4.73e8 \\
infomap+PASCO $(t=10)$ & 1.18e3 & 0.546 & 0.841 & 0.872 & 4.66e8 \\
infomap+PASCO $(t=15)$ & 1.50e3 & \textbf{0.553} & 0.85 & 0.882 & 4.64e8 \\
\bottomrule
\end{tabular}
\end{table}

\end{appendices}


\end{document}